\documentclass{article}
\pdfoutput=1

\usepackage[preprint,nonatbib]{neurips_2023}
\usepackage[numbers]{natbib}

\usepackage[utf8]{inputenc} %
\usepackage[T1]{fontenc}    %
\usepackage{hyperref}       %
\usepackage{url}            %
\usepackage{booktabs}       %
\usepackage{amsfonts}       %
\usepackage{microtype}      %
\usepackage[dvipsnames]{xcolor}   %

\usepackage{listings}
\definecolor{codegreen}{rgb}{0,0.6,0}
\definecolor{codegray}{rgb}{0.5,0.5,0.5}
\definecolor{codepurple}{rgb}{0.58,0,0.82}
\definecolor{backcolour}{rgb}{0.95,0.95,0.92}
 \definecolor{deepblue}{rgb}{0,0,0.5}
\definecolor{deepred}{rgb}{0.6,0,0}
\definecolor{deepgreen}{rgb}{0,0.5,0}

\lstdefinestyle{mypython}{
    backgroundcolor=\color{backcolour},   
    commentstyle=\color{codegreen},
    keywordstyle=\color{magenta},
    emph={self},  
    emphstyle=\color{deepred},    %
    numberstyle=\tiny\color{deepblue},
    stringstyle=\color{codepurple},
    basicstyle=\scriptsize\ttfamily,
    breakatwhitespace=false,        
    xleftmargin=0pt,
    breaklines=true,                 
    captionpos=b,                    
    keepspaces=true,                 
    numbers=left,                    
    numbersep=3pt,                  
    showspaces=false,                
    showstringspaces=false,
    showtabs=false,                  
    tabsize=2,
    language=Python,
    morekeywords={with},              %
}

\definecolor{neuralbackcolour}{rgb}{0.9375, 0.96875, 0.9960}
\lstdefinestyle{neural}{
    backgroundcolor=\color{neuralbackcolour},   
    commentstyle=\color{codegreen},
    keywordstyle=\color{magenta},
    numberstyle=\tiny\color{deepblue},
    stringstyle=\color{codepurple},
    basicstyle=\scriptsize\ttfamily,
    breakatwhitespace=false,
    xleftmargin=0pt,
    breaklines=true,                 
    captionpos=b,                    
    keepspaces=true,                 
    numbers=left,                    
    numbersep=3pt,                  
    showspaces=false,                
    showstringspaces=false,
    showtabs=false,                  
    tabsize=2,
    language=Python
}
\usepackage{wrapfig}
\usepackage{amsmath}
\usepackage{amssymb}
\usepackage{mathtools}
\usepackage{amsthm}
\usepackage{multirow}
\usepackage{siunitx}
\usepackage{tcolorbox}
\usepackage{enumitem}
\usepackage[tableposition=top]{caption}
\usepackage{subcaption}
\usepackage{wrapfig}
\usepackage{soul}

\theoremstyle{plain}
\newtheorem{theorem}{Theorem}[section]

\theoremstyle{definition}
\newtheorem{definition}[theorem]{Definition}

\theoremstyle{remark}

\usepackage[framemethod=tikz]{mdframed}

\newcommand{\jasonhidden}[1]{}

\newcounter{subproblemno}
\newcommand{\subproblem}[1]{\refstepcounter{subproblemno}\label{#1}}

\title{Neuro-Symbolic Execution of Generic Source Code}

\author{%
 Yaojie Hu \\
  Department of Computer Science\\
  Iowa State University\\
  \texttt{jhu@iastate.edu} \\
  \And
  Jin Tian \\
  Department of Computer Science \\
  Iowa State University\\
  \texttt{jtian@iastate.edu} \\
}

\begin{document}

\maketitle

\begin{abstract}
Can a Python program be executed statement-by-statement by neural networks %
composed according to the source code?
We formulate %
the Neuro-Symbolic Execution Problem and introduce Neural Interpretation (NI), 
the first neural model for the execution of generic source code that allows missing definitions. NI preserves source code structure, where every variable has a vector encoding and every function executes a neural network. %
NI is a novel neural model of computers with a compiler architecture that can assemble neural layers ``programmed'' by source code.
NI is the first neural model capable of executing Py150 dataset programs, including library functions without concrete inputs, and it can be trained with flexible code understanding objectives.
We demonstrate white-box execution without concrete inputs for variable misuse localization and repair.
\end{abstract}

\section{Introduction}

Learning to model computation has fundamental importance. Ideally, neural models of computers should read a snippet of source code and execute it step-by-step. Various models of computers have been proposed, but limitations remain: 
\begin{itemize}[topsep=-5pt,itemsep=-1ex,partopsep=1ex,parsep=1.2ex, leftmargin=2ex]
    \item Large language models (LLM) can be trained  \cite{zaremba2014learning, DBLP:journals/corr/abs-2108-07732} or prompted \cite{ouyang2022training} to predict the output of a program, but they lack the inductive biases to reflect step-by-step execution or function composition.
    \item LLMs can learn to predict step-by-step execution traces, but they may require concrete inputs and cannot execute function definitions without call-sites \cite{nye2022show}.
    \item Some neural networks can learn domain-specific programs \cite{bieber2020learning, chen2021latent}, but only for certain domains.
    \item Algorithms \cite{velivckovic2019neural} and computer architectures \cite{graves2014neural} can be modeled directly by neural networks without source code, but, as a result, they are not ``programmable'' by source code. %
    \item Symbolic execution and abstract interpretation are symbolic models in program analysis \cite{baldoni2018survey}, but they are often not flexible nor expressive for source code in the wild and require manual modeling.
\end{itemize}
Due to these limitations, no neural execution model today can execute \textit{generic} source code.
On the other hand, %
humans can \textit{mentally execute} generic source code robustly without any of these limitations.
One can simulate the program's execution step-by-step top-to-down and follow function compositions (Fig.~\ref{fig:mental}).
When variables and functions are used without contextual definitions, we can guess the meaning based on their names. 
Humans do not execute the program exactly: we may identify an infinite loop from the source code without simulating it infinitely.
The key is to be abstract: partial source code may not be executed concretely, but they are still meaningful,
and humans mentally execute to reason about abstract properties and information (\textit{abstract semantics}), rather than concrete execution traces that are often impossible to predict for generic source code.

\begin{wrapfigure}[8]{R}{0.45\textwidth}
\vspace{-16pt}
\begin{lstlisting}[style=mypython]
# Python
def celsius_to_fahrenheit(celsius):
    fahrenheit = (celsius * 1.8) + 32
    # what if it says "return celsius"?
    return fahrenheit
\end{lstlisting} %
\vskip -1em
\caption{Humans ``mentally execute'' each statement from top to down and infer the meaning of variables by the names.}
\label{fig:mental}
\end{wrapfigure}

This paper aims to develop a neural model of computers capable of executing generic source code like human mental execution, overcoming the limitations above. First, we formalize the problem as the Neuro-Symbolic Execution Problem.
We then introduce Neural Interpretation (NI), a neural model of computers that reads (Python) source code and executes neural instructions to perform abstract reasoning about the program. 
Every function is represented by a neural network, and every variable is represented by a vector, encoding their semantics. %
During execution, instead of a real function call on its arguments, NI executes the neural model of the function on the vector variables.
An external neural memory table associates variables with vector values by their names, the first string-vector representation of computer memory to the best of our knowledge. 
The overall architecture of NI follows compiler design, in order to decompose the source code and assemble neural layers statement-by-statement. 
NI learns to model program semantics such that
given the semantics of a function and its arguments, NI can predict the semantics of the result.

As Neural Interpretation mirrors real execution of a generic symbolic program, it is a step toward building
a unifying bridge between neural and symbolic computation for tomorrow's hybrid computers \cite{marcus2018deep,  goertzel2007artificial}. Symbolic reasoning takes many forms. 
Computer programs in Turing-complete languages are the most general form of symbolic reasoning \cite{church1936unsolvable,turing1938computable}, and source code is used to express rules about everything computable in life,  which NI learns. 
As a neuro-symbolic model \cite{garcez2020neurosymbolic, santoro2021symbolic}, NI is not a black-box \cite{rudin2019stop} and enjoys many symbolic desiderata by design, such as explainability \cite{das2020opportunities, arrieta2020explainable}, factorization/composition \cite{bengio2017consciousness, DBLP:journals/corr/abs-2012-05208}, and chained reasoning \cite{golovneva2022roscoe, shindo2021neuro, nye2022show, wei2022chain}.

Our contributions are summarized as follows: 
\begin{itemize}[topsep=-5pt,itemsep=-1ex,partopsep=1ex,parsep=1.2ex, leftmargin=2ex]
    \item We formulate the Neuro-Symbolic Execution Problem to identify the essential properties 
    for executing generic source code compositionally with learned neural representations.
    \item We introduce Neural Interpretation, the first neural model of computers that reads general-domain source code and executes neural statements step-by-step even with a partial context or missing definitions.
    \item We design the first compiler-architecture neural model of computers with the first string-vector hash map external memory that breaks away from the attention paradigm. %
    \item We propose the Guesser and the Executor that use large language models to predict vector representations for variables and neural networks for functions, the first usage of currying to predict neural networks as model outputs. %
    \item We demonstrate that Neural Interpretation can learn flexibly-defined abstract semantics and, for the first time, perform white-box execution of function definitions without call-sites, demonstrated on a variable misuse localization and repair dataset. %
\end{itemize}
Additionally, we have marked \textbf{sub-problems ($\texttt{SP}$)} encountered, solved, or to be solved. With detailed road map in Appx. \ref{appx:roadmap}, we hope to inspire future approaches. Code and data will be publicly available. %

\section{Related Work}

\paragraph{Symbolic execution and abstract interpretation.} Symbolic execution \cite{king1976symbolic} and abstract interpretation \cite{cousot1977abstract} %
use interpreters to build symbolic models for program analysis \cite{cadar2008klee, wei2019staged}. Symbolic execution uses placeholder symbols to represent variables, and abstract interpretation uses lattice nodes of possible values.
As symbolic models, the two are well-founded on theories, but they do not always work for generic programs in the wild due to common challenges such as dynamic typing \cite{sapra2013finding}. 
They cannot learn from meaningful variable names, and although manual engineering can help, sophisticated methods are needed to handle basic constructs such as floating point numbers \cite{cousot2001abstract, titolo2018abstract}. Lastly, both models face state space explosion problems during reasoning \cite{schmidt1998program, baldoni2018survey}.

\paragraph{Neural models of source code execution.}

Sequence-to-sequence (Seq2Seq) models can directly predict execution outputs from source code \cite{zaremba2014learning, DBLP:journals/corr/abs-2108-07732}. Although Seq2Seq formulation may model Turing-complete programs, function composition and step-by-step execution are not explicit, reflecting the general black-box problem of neural networks \cite{rudin2019stop}.  

To address the black-box problem, models may predict concrete execution traces step-by-step, e.g. Show Your Work with LLMs \cite{nye2022show}.
However, source code in the wild are usually library functions without concrete entry-point inputs (e.g. Fig \ref{fig:mental}), which cannot be modeled. %

Execution of domain-specific source code can be learned, including arithmetic programs \cite{chen2021latent}, graph algorithms \cite{bieber2020learning}, or specific assembly programs for compiler optimization \cite{shi2019learning}.
 However, they do not model general-domain programs in the wild.
 
A Neural Compiler can compile PASCAL code to neural networks through cellular code \cite{gruau1995neural}.

\paragraph{Neural models of computers without source code.} Neural models of computers have been proposed \cite{graves2014neural, graves2016hybrid}, often with a controller network that emulates CPU and attention mechanism for memory, while our paper proposes a novel string-vector hash map memory. Execution of selected algorithms can be learned \cite{velivckovic2019neural, yan2020neural}. But this type of models cannot be ``programmed'' by source code. 

\paragraph{Neuro-symbolic reasoning.} Neuro-symbolic reasoning aims to integrate symbolic reasoning and deep learning \cite{hitzler2022neuro, garcez2020neurosymbolic, marcus2018deep}. 
Domain-specific programs in non-Turing-complete languages can be synthesized by neural networks and executed externally \cite{mao2018the,chen2021spreadsheetcoder}, with the recent highlight to augment language models with tools \cite{mialon2023augmented, schick2023toolformer}. But learning execution is more difficult than synthesis \cite{DBLP:journals/corr/abs-2108-07732}.

\section{The Neuro-Symbolic Execution Problem}

Neuro-Symbolic Execution (NSX) aims to learn an isomorphic mapping between the source code space and neural representation space (Figure \ref{fig:iso}). Through NSX Isomorphism, every source code function can be modeled by a neural network, and every variable can be modeled by a latent vector. The structure of the source code is preserved through the isomorphism, such that the return variable of a source code function execution maps to the result of its corresponding neural network executed on the corresponding argument vectors. That is, executions on the two spaces correspond bijectively.

\begin{wrapfigure}[10]{R}{0.3\textwidth}
\vspace{-17pt}
    \centering
    \includegraphics[width=\linewidth]{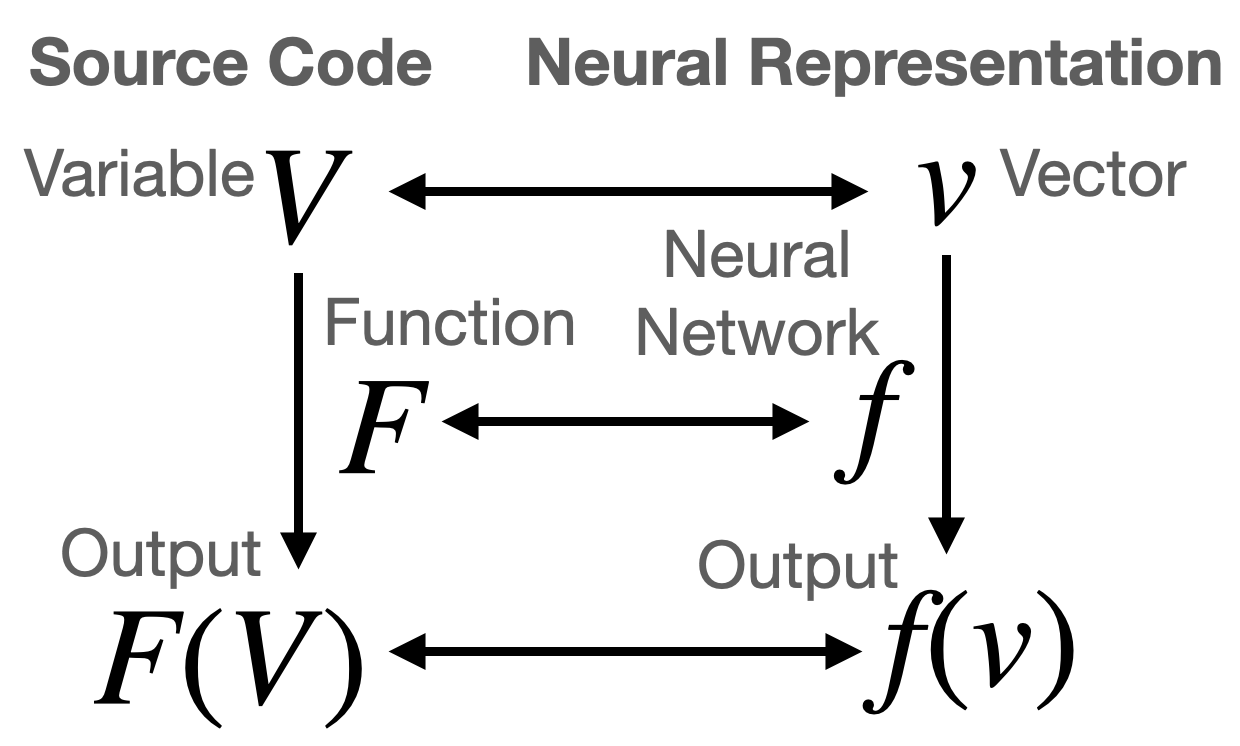}
    \caption{Neuro-Symbolic Execution learns an isomorphism to model source code.}
    \label{fig:iso}
\end{wrapfigure}

\begin{definition}[Neuro-Symbolic Execution Isomorphism, NSXI] \label{def:nsxi}
Let $\mathbf V =\{V_1, V_2, \ldots,  V_N\}$ be the variables and $\mathbf F =\{F_1, F_2, \ldots, F_M\}$ be the functions in some source code $S$, where every function $F_m\in \mathbf F$ has input variables $\mathbf I_m \subseteq \mathbf V$ and output variables $\mathbf O_m \subseteq \mathbf V$. Let $\mathbf v = \{v_1, v_2, \ldots, v_N\}$ be a set of vectors and $\mathbf f =\{f_1, f_2, \ldots, f_M\}$ be a set of neural networks. Let $\alpha$ be a bijective function $\mathbf V \to \mathbf v$ and $\beta$ be a bijective function $\mathbf F\to \mathbf f$. We say $(\mathbf V, \mathbf F)$ and $(\mathbf v, \mathbf f)$ are isomorphic under bijections $\alpha, \beta$ if for every function $F_m \in \mathbf F$ with $F_m(\mathbf I_m)=\mathbf O_m$, we have $f_m(\mathbf i_m)=\mathbf o_m$, where $f_m= \beta(F_m)$, $\mathbf i_m = \alpha(\mathbf I_m)$, $\mathbf o_m= \alpha(\mathbf O_m)$.
\end{definition}

We will call vector $v = \alpha(V)$ the abstract semantics of variable $V$, and neural network $f=\beta(F)$ the abstract semantics of function $F$, inspired by the homonyous concepts in Abstract Interpretation \cite{cousot1977abstract}. %

To learn bijections $\alpha, \beta$, we may fine-tune existing language models \cite{radford2018improving} to create neural representations \cite{mikolov2013distributed} for source code variables and functions by inferring the information they carry and the calculation they perform based on their names. In reverse, we may train classifiers $\alpha^{-1}, \beta^{-1}$ to take a neural representation and classify the source code variable or function. Auto-encoding is a similar task \cite{montero-etal-2021-sentence}, but the additional condition is critical: as $\alpha,\beta$ are structure-preserving mappings, the execution outputs in both spaces need to correspond as well. 
\begin{definition}[Neuro-Symbolic Execution Problem, NSXP]\label{def:NSXP}
The NSXP is to learn bijections $\alpha, \beta$ such that NSXI holds under $\alpha, \beta$ for every source code $S$ in the support of a distribution $\mathbf S$. %
\end{definition}
We will call a solution to the NSXP an NSX model.
Isomorphism serves as a mathematical tool to describe the theoretical ideal of a neural model that faithfully represents source code: bijection implies that the representation is lossless, and isomorphism preserves the execution structure. 
From a computation perspective, the world consists of data and its transformation, and NSX learns to the computational relationships among language concepts.

Once $\alpha, \beta$ are learned, we can model source code execution with neural execution from top to down and, at every step, obtain a latent representation that describes the output variable. That is, Neuro-Symbolic Execution follows source code compositionality.

\begin{theorem}[Compositionality of Neuro-Symbolic Execution]
For an NSX model, let $F, G$ be two source code functions and $A$ be a variable. Then, a compositional function call $G(F(A))$ returns a value with abstract semantics $g(f(a))$, where $a,f,g$ are abstract semantics of $A, F, G$, respectively.
\end{theorem}

\begin{wrapfigure}[19]{R}{0.41\textwidth}
\vspace{-20pt}
    \centering
\begin{lstlisting}[style=neural]
# Neural Interpretation
# line 2, initialize parameter semantics, i.e. guessing
v1 = initialize_vector("celsius")
store_memory("celsius", v1)
# line 3, compute fahrenheit
v2 = lookup_memory("celsius")
v3 = initialize_vector("1.8")
v4 = neural_multiply(v2, v3)
v5 = initialize_vector("32")
v6 = neural_add(v4, v5)
store_memory("fahrenheit", v6)
# line 5, return statement
v7 = lookup_memory("fahrenheit")
return v7
\end{lstlisting}
\vskip -0.5em
\caption{A pseudocode example of Neuro-Symbolic Execution for the \texttt{celsius\_to\_fahrenheit} program in Fig.~\ref{fig:mental}. Neural Interpretation resembles how a computer executes a program given the source code. Neural Compilation signature: $(v_1, v_7)$.}
\label{fig:simple}
\end{wrapfigure}

Compositionality of functions is a fundamental property of computer programs, and Neuro-Symbolic Execution mirrors the entire code structure executed. Notice that, for example, source code $G \circ F$ would define a neural network $g \circ f$, and a change in the order of execution to $F\circ G$ composes an alternative network $f\circ g$. 
For NSX models such as our proposed Neural Interpretation, source code parameterizes the neural model: neural layers representing the statements are ``programmed'' according to the source code.
From this perspective, NSX models are analogous to programmable computers, while specialized algorithm models \cite{velivckovic2019neural} are analogous to early computers that cannot be programmed \cite{atanasoff1984advent}.
Lastly, notice that NSX is more general than concrete execution in the sense that NSX can handle concrete inputs, but concrete execution cannot run if any definition is missing. 

\paragraph{A simple example.} Fig.~\ref{fig:simple} illustrates the pseudocode of Neuro-Symbolic Execution for the \texttt{celsius\_to\_fahrenheit} program in Fig.~\ref{fig:mental}. To preview, Neural Interpretation is our method to achieve Neuro-Symbolic Execution, where \texttt{initialize\_vector} for variables such as \texttt{celsius} is ``guessed'' by the Guesser neural network, and functions such as \texttt{neural\_multiply} are executed by the $\lambda$-Executor neural network.

\paragraph{A neural model with a compiler architecture.} 
How do we build a Neuro-Symbolic Execution model? Translating source code symbol by symbol into neural representations reminds us of what compilers do, except that the translation target is not another lower-level programming language but some ``neural instructions''. Indeed, compiler architecture is a natural choice for NSX models, and Neural Interpretation is the first neural model of computers with a compiler architecture.

\section{Neural Interpretation}

\begin{wrapfigure}[3]{R}{0.6\textwidth}
    \centering
    \raisebox{0pt}[\dimexpr\height-5\baselineskip\relax]{
    \includegraphics[width=\linewidth]{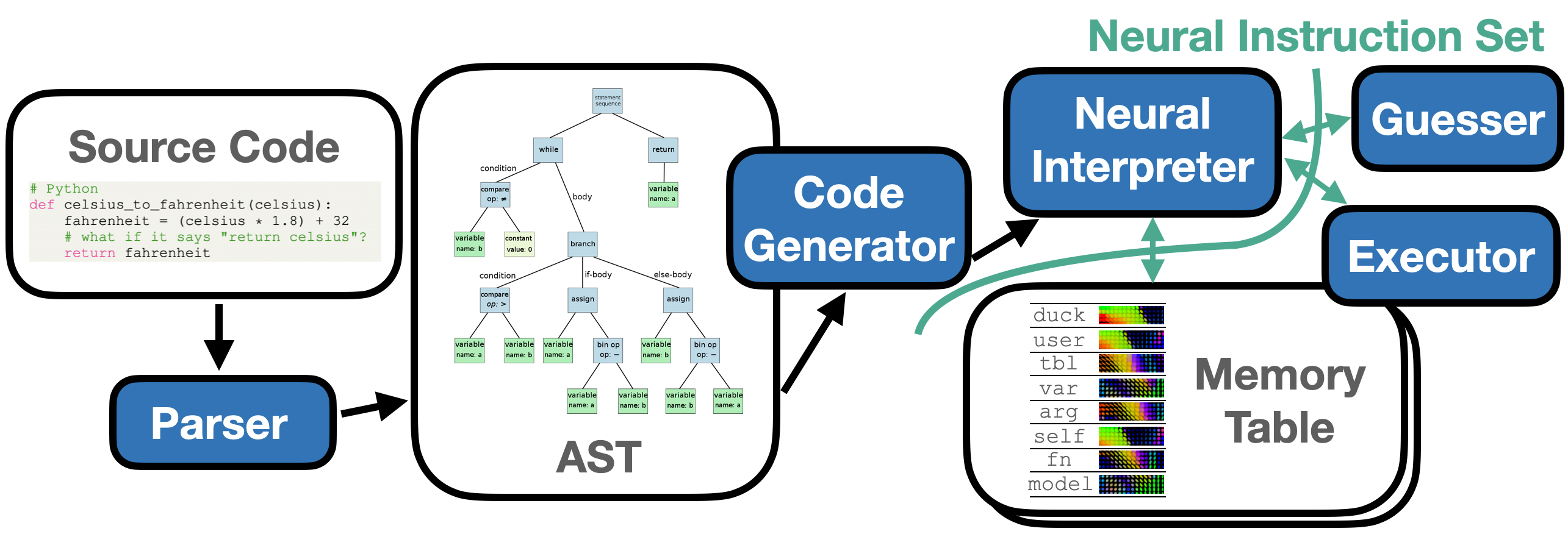}
    }
    \vskip -0.6em
    \caption{The overall architecture of Neural Interpretation.}
    \label{fig:archi}
\end{wrapfigure}

We introduce Neural Interpretation (NI), an NSX model.

\subsection{Architecture Overview}

Neural Interpretation has a compiler/interpreter architecture to parse and model source code compositionally. A compiler translates a higher-level language source code to a lower-level code executed on a machine \cite{aho2007compilers}. NI is special as the lower-level language is a Neural Instruction Set (NIS, Tab. \ref{tab:nis}) executed by a Neural Interpreter. Without loss of generality, the higher-level language of NI is Python, a Turing-complete language for high-level programming.
Other languages can be added (sub-problem \subproblem{sp:new_language}\texttt{SP-\ref{sp:new_language}}, Appx. \ref{appx:language}).

Neural Interpretation has four major components (Fig.~\ref{fig:archi}), following a compiler architecture \cite{aho2007compilers}:
\begin{enumerate}[topsep=-5pt,itemsep=-1ex,partopsep=1ex,parsep=1ex, leftmargin=3ex]
    \item A \textit{Parser} parses a Python source code into an abstract syntax tree (AST). 
    We use an existing open-source parser called Tree-sitter \cite{brunsfeld}.
    \item A \textit{Code Generator} traverses an AST and invokes the Neural Interpreter to perform the language construct (e.g. while-loops, lists, function calls) of each AST node (Python statement). 
    \item A \textit{Neural Interpreter} is the execution environment (virtual machine). The Neural Interpreter keeps track of the variable scopes where each variable has a vector representation. The Neural Interpreter also keeps track of a stack of control flows. When the Code Generator encounters a Python function call, the Neural Interpreter executes a neural network on the vector representations. 
    \item The \textit{Guesser} and \textit{Executor} are two large language models for the abstract semantics of programs. The Guesser initializes vector representations. The Executor models a neural network that can represent Python functions.
\end{enumerate}

Parser and Code Generator handle the syntax and translate a source code (e.g. Fig.~\ref{fig:mental}) into an intermediate code of NIS instructions (e.g. Fig.~\ref{fig:simple}). %
The Neural Interpreter executes NIS instructions with neural abstract semantics.
The Code Generator is designed specifically for Neural Interpretation, but it mostly follows the standard compiler architecture \cite{aho2007compilers} and Python language reference \cite{vanrossum1995python}. NI is multi-threaded for batch execution on GPUs. Engineering details can be found in Appx. \ref{appx:engineering}, including code generation routines for example Python constructs.

\subsection{Two Design Principles}

We anticipate two major sub-problems and introduce the system design principles to address them.

\paragraph{The loop problem and the linearity principle.}
How does Neural Interpretation handle loops (\subproblem{sp:loop}\texttt{SP-\ref{sp:loop}})? One option is to predict the number of times a loop will be executed, and then neurally execute the code block that many times. However, it is not always possible (e.g. Fig. \ref{fig:loopgran}). Another issue is that, if a network predicts that a loop will be executed $L$ times, the prediction $L$ is not differentiable with respect to the result of the execution. Additional learning algorithms such as policy gradient \cite{williams1992simple} may be needed to address this problem, which would make NI overly complicated as a first baseline solution. Alternative designs for the loop problem (\texttt{SP-\ref{sp:loop}}) are explored in Appx. \ref{appx:roadmap}.

\begin{wrapfigure}[10]{R}{0.35\textwidth}
\vspace{-10pt}
    \centering
\begin{lstlisting}[style=mypython]
# Python
def a_loop(total_iter):
    lst = []
    for i in range(total_iter):
        lst.append(i)
    return lst
\end{lstlisting} %
\vskip -0.5 em
    \caption{An example Python program that reflects the loop problem and the list problem.}
    \label{fig:loopgran}
\end{wrapfigure}

When humans read a loop statement, we reason abstractly about the loop condition in relation to the variables that would determine the number of iterations. If we step through the loop in our heads to understand it, once or twice can often suffice.
Given this intuition, we believe NI can sufficiently model the semantics of a loop by adopting the following principle.
\begin{definition}[Linearity principle]
An Neuro-Symbolic Execution model satisfies the linearity principle if it executes every statement in the source code once,
and when a control flow is encountered, each of its branch is executed once.
\end{definition}
Following the linearity principle, NI does not loop, as it only executes a loop block once. When a control flow is encountered, the control flow condition is neurally executed to predict a context vector. The context vector(s) is given to the Executor to represent the control flow of a call. %
By adopting the principle, we assume a single statement execution in a control flow can obtain the correct semantics given the contextual encoding.
As a result, NI has linear time complexity with respect to the length of the source code, formally stated in Appx. \ref{appx:complexity}. The code generation routine is given in Appx. \ref{appx:engineering}.

\paragraph{The list problem and the named representation principle.}
How does Neural Interpretation represent lists (or sets, tuples, etc.) (\subproblem{sp:gran}\texttt{SP-\ref{sp:gran}})?
One option is to represent every element in a list with a vector. This would require us to know the length of the list ahead of time, which is not always possible (e.g. Fig.~\ref{fig:loopgran}). Moreover, the memory consumption would scale with the length of the list, which is not ideal if all elements have similar abstract semantics that could be summarized succinctly with one representation. 
We believe NI can represent the semantics of a list with one vector and adopt the named representation principle.
\begin{definition}[Named representation principle]
An Neuro-Symbolic Execution model follows the named representation principle if its memory holds a constant-length vector representation for every variable that has a name and has been encountered during the execution.
\end{definition} 
For example, if \texttt{lst} is a list, it has a vector representation by the variable name ``lst''. \texttt{lst[0]}, the first element of \texttt{lst}, has a derived name, so it does not have its own vector representation stored in the Interpreter memory. When a representation for \texttt{lst[0]} is needed, the Executor computes it on-demand with a built-in \texttt{list\_index} function (see Appx.~\ref{appx:build_in}).  
By adopting the principle, we assume that a single vector can summarize the commonality of elements in a list.
As a result, NI has linear space complexity, formally stated in Appx. \ref{appx:complexity}.
The code generation routine is in Appx. \ref{appx:engineering}.

Note that the space and time complexity of NI does not follow the space and time complexity of the actual execution. To see this, consider an infinite loop that appends an element to a list at every iteration (Fig.~\ref{fig:loopgran}): the time and space complexity would be unbounded w.r.t. source code length. Like humans' mental execution, Neural Interpretation reasons abstractly about the semantics of a program.

\subsection{Neural Interpreter and Neural Instruction Set}
In Neural Interpretation, we introduce a Neural Instruction Set (NIS) as a complete interface for the Neural Interpreter to invoke neural computations and representations (Tab. \ref{tab:nis}). 
The Code Generator iterates through the AST and performs Neural Instructions based on the type of the AST node. Normal CPU instruction sets could contain many low-level instructions (e.g. x86-64 \cite{intel2013intel}). The success of NIS shows that high-level computational reasoning may only require a small set of building blocks.

\begin{wraptable}[11]{R}{0.55\textwidth}
\vspace{0pt}
    \caption{A high-level Neural Instruction Set (NIS) exposes neural representations to the Neural Interpreter.}
    \label{tab:nis}
    \centering
    \resizebox{\linewidth}{!}{
    \def\arraystretch{1.2}
    \begin{tabular}{p{0.5in} p{2.8in}}
        \toprule
        Instruction & Summary \\
        \midrule
        \texttt{lookup} & Lookup a vector by name in all available scopes. \\
        \texttt{store} & Store a (variable name, vector) pair to the top scope. \\
        \texttt{guess} & Provided by Guesser, initializes the vector representation for each variable identifier from source code. \\
        \texttt{lambda} & Provided by Executor, executes the neural network representing a Python function. \\
        \bottomrule
    \end{tabular}
    }
\end{wraptable}

The Neural Interpreter has a stack of neural memory tables that keep track of the current variable scopes per Python scoping rules. Each variable has a vector representation, accessed through \texttt{lookup} and \texttt{store} instructions. Instead of a heap or a stack often used by a low-level execution environment, the variables in the Interpreter are stored in a hash table with string keys representing their names and vector values representing their current vector semantics.
The \texttt{guess} and \texttt{lambda} instructions will be defined next with the Guesser and Executor networks. 

If the Neural Instruction Set has no neural components, our Neural Interpreter may become a normal interpreter instead. NIS provides a necessary and complete neural interface for 
Neuro-Symbolic Execution of generic programs. 
We believe that languages other than Python can reuse the same NIS and possibly even the same learned representations, because program semantics can be similar while syntaxes differ across languages (\texttt{SP-\ref{sp:new_language}}). Alternative NIS design is possible (\subproblem{sp:nis}\texttt{SP-\ref{sp:nis}}).
 
\subsection{Guesser to Initialize Representations}

Guesser is a neural network that predicts a vector for each expression (and identifier) in the source code (Fig. \ref{fig:guesser}). For a variable, the guessed vector initializes its abstract semantics. For a function, the guessed vector is the signature vector that parameterizes its abstract semantics network (Eq. \ref{eq:currying}).

\begin{wrapfigure}[12]{R}{0.45\textwidth}
\vspace{-20pt}
\centering
\begin{mdframed}[backgroundcolor=backcolour,leftmargin=0cm,hidealllines=true,%
  innerleftmargin=0cm,innerrightmargin=0cm,innertopmargin=-0.72cm,innerbottommargin=-0.10cm]
\begin{lstlisting}[style=mypython, escapechar=\%]]
# Python
def %\underline{celsius\_to\_fahrenheit}%(%\underline{celsius}%):
    %\underline{fahrenheit} = \underline{\underline{(\underline{\underline{celsius} * \underline{1.8}})} + \underline{32}}% 
    # what if it says "return celsius"?
    return %\underline{fahrenheit}%
\end{lstlisting} %
\end{mdframed}
    \vskip -1.2em
\caption{A guesser predicts a vector representation for every expression (including identifiers), indicated by an \underline{underline} in this example.
The same variable at different locations can have different representations, similar to the static single-assignment form \cite{cytron1989efficient, appel1998ssa, ananian2001static}.%
}
    \label{fig:guesser}
\end{wrapfigure}

Formally, given a source code string $S$ that contains expressions (including identifiers) $E_1, E_2, ... E_N$, the Guesser $\gamma$ outputs $N$ vectors of length $H$:%
\begin{align}
    \mathcal \gamma(S) = [e_1, e_2, ... e_N]. \label{eq:guesser}
\end{align}
Each vector $e_n$ initializes the guessed abstract semantics of the corresponding identifier or expression. The Guesser $\gamma$ implements \texttt{guess} in the NIS (Tab. \ref{tab:nis}).

We implement the Guesser with Transformer architecture \cite{vaswani2017attention, devlin2018bert}.
The Transformer outputs one vector encoding for each token in the source code string, and a guessed vector is produced by max-pooling all corresponding token encodings \cite{gage1994new, kudo2018sentencepiece}. Pooling occurs on-demand as the Executor uses an argument, and if the argument has a previously executed vector, the guessed vector is concatenated to enrich Executor input in practice. Details about Guesser pooling are provided in Appx. \ref{appx:engineering}.

Neural Interpretation is about execution, but guessing (a Guesser-like component) is unavoidable to model generic source code with missing definitions, context, and library dependency. NI is a program (Fig.~\ref{fig:simple}), and like all programs, variables need to be initialized before use (for more, see Appx. \ref{appx:missing}).

\subsection{$\lambda$-Executor to Represent and Execute Functions}

The Executor learns to model a function $F$ by outputting its abstract semantics $f$. By NSX Isomorphism, given the abstract semantics $f$ of the function $F$ and $a$ of the arguments $A$, the abstract semantics of the return variable should be correctly modeled as $f(a)$.

Neural Interpretation represents every Python function with a neural network parametrized by a signature vector. How do we represent a Python function with a unique neural network, especially as the potential %
number of Python functions is infinite (\subproblem{sp:infinite_fns}\texttt{SP-\ref{sp:infinite_fns}})? What happens if the function definition is missing from source code (\subproblem{sp:missing_fn_def}\texttt{SP-\ref{sp:missing_fn_def}})? We propose the $\lambda$-Executor to address both issues.

\paragraph{$\lambda$-Executor predicts neural networks through currying.}
Our Executor implementation, $\lambda$-Executor, is a neural network that predicts another neural network $f$ given vector signature $\theta_f$, where $f$ represents function $F$ as its abstract semantics through currying \cite{curry1958combinatory}. 
Formally, $\lambda$ is a neural network that takes signature $\theta_f$ of $F$ and abstract semantics $a$ of arguments $A$ and predicts the returned abstract semantics $r$ for $F(A)$:
\begin{align}
    \lambda(\theta_f, a) = r.
\end{align}
$\lambda$ can output function $f$ from $\theta_f$ through currying:
\begin{align}
    f & = \lambda(\theta_f), \label{eq:currying}\\
    f(a) & = \lambda(\theta_f, a)=r. \label{eq:return}
\end{align}
$\lambda$-Executor, a single neural network, can predict a neural network $f$ given any Python function $F$. 
If $F$ is not defined in the source code, $\theta_f$ is obtained from the guessed representation. As a result, neural models of Python functions are predicted on-demand and share parameters efficiently.

We implement $\lambda$-Executor with a Transformer architecture. A sequence input containing $\theta_f$, $a$, and control flow context is fed into the Transformer. Input $\theta_f$ achieves currying with a role similar to task prefixes of T5 Transformers \cite{raffel2020exploring}.
$\lambda$-Executor implements \texttt{lambda} in the NIS (Tab. \ref{tab:nis}). For more on Executor, see Appx. \ref{appx:engineering}.

\paragraph{Neural compilation of function definitions.}
How does NI represent functions if their definitions are available in the source code (\subproblem{sp:have_fn_def}\texttt{SP-\ref{sp:have_fn_def}}, e.g. Fig.~\ref{fig:mental})? One approach would be to substitute the parameters with call-site arguments and execute the definition line-by-line, like a real Python interpreter or inline expansion \cite{serrano1997inline}. However, this would break the linearity principle, as a line in a function definition for $F$ could be executed twice if $F$ is called twice. Also, how would recursive functions be executed?

If a function $F$ is defined in the source code, NI translates it once to its signature $\theta_f=(p,r)$ %
through \textit{neural compilation}: perform Neural Interpretation on $F$ starting from its parameters $P$ to its return variable $R$ and store the signature $\theta_f = (p, r)$ in the Interpreter memory to represent $F$. When $F$ is invoked, compiled signature $\theta_f$ plugs into $\lambda$-Executor to obtain $f$ by currying (Eq. \ref{eq:currying}). To neurally compile a recursive function $F$, function calls of $F$ within definitions of $F$ use guessed abstract semantics, as if the definition is absent. NI's neural compilation adheres to the linearity principle and translates a function definition to a vector signature.

\subsection{Training Neural Interpretation}
Methods introduced so far perform Neural Interpretation of generic source code. What is the ground truth to train NI (\subproblem{sp:ground_truth}\texttt{SP-\ref{sp:ground_truth}})? 
We believe that language modeling objectives are not gold standards to train NI, as the same variable can carry different information at different locations (see Appx. \ref{appx:ground}).
Instead, we may design loss functions to express NSXI and other code understanding objectives. As shown in Sec \ref{sec:learning} and \ref{sec:misuse}, NI is general-purpose and can learn different code understanding objectives flexibly.

\section{Learning to Execute Programs Abstractly}

As a general neural program execution framework, Neural Interpretation has great potential to be applied to a variety of program understanding problems.
We aim to demonstrate the essential results such as differentiable training, flexible objectives, and white-box execution, in order to establish the first baseline Neuro-Symbolic Execution model.

\subsection{Learning to Model the Execution of Generic Python Code}

\label{sec:learning}

\paragraph{A generic Python dataset.} The ETH Py150 Open dataset (Py150) \cite{kanade2020learning, raychev2016probabilistic} consists of generic Python scripts from open-source Github projects. As most programs in the wild are library methods, they often have no concrete entries such as unit tests, and they are often not self-contained due to dependency requirements. For this reason, concrete execution traces may be unobtainable, and there have been no successful execution models for Py150 programs. Neural Interpretation is the first neural model to demonstrate successful execution on Py150 dataset, due to its unique ability to use robust and generic abstract semantics representations in lieu of concrete traces, opening up the new possibility of execution-based learning on large generic code corpus.

The training, validation, and test sets of ETH Py150 Open dataset contain 74749, 8303, 41457 programs. We filter out around 13.3\% programs longer than 10000 characters, to prevent out-of-memory errors as a standard practice \cite{devlin2018bert, dong2019unified}, and around 1.5\% programs with code generation errors 
when NI encounters some edge cases of Python language constructs that we chose not to implement in order to reduce engineering complexity. See Appx. \ref{appx:dataset} for filtering and truncation statistics.

\begin{wrapfigure}[12]{R}{0.6\textwidth}
\vspace{-15pt}
    \centering
    \includegraphics[width=\linewidth]{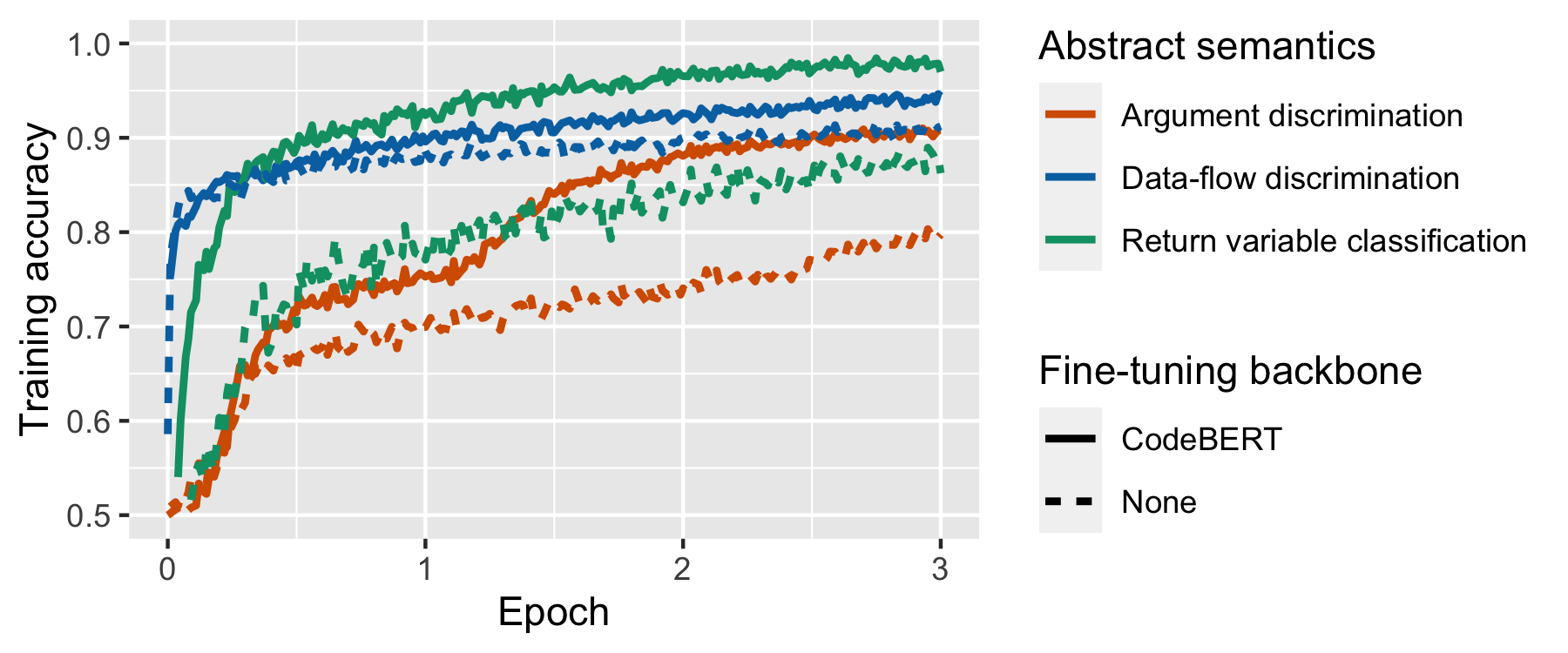}
    \vskip -0.7em
    \caption{Training accuracy versus epochs while Neural Interpretation learns the abstract semantics objectives on the ETH Py150 Open dataset.}
    \label{fig:three}
\end{wrapfigure}

\paragraph{Abstract semantics selection.} 
NI can flexibly encode different program properties through abstract semantics. We select three important code understanding properties related to NSXI to train NI on Py150.
NI is performed for the entire source code snippets in a batch before loss is computed on applicable objects. The overall loss $L$ is the sum of the following cross-entropy losses with $L=L_1+L_2+L_3$. Detailed loss definitions are in Appx. \ref{appx:loss}.

\textit{Return variable classification loss $L_1$.} NI is trained to classify the return variable of an executed statement, as required by Neuro-Symbolic Execution Isomorphism (Def. \ref{def:nsxi}). For example, \texttt{celsius\_to\_fahrenheit(25)} is likely to be assigned to variables called \texttt{fahrenheit} or \texttt{converted}, but not \texttt{duck} or \texttt{table}. 
Let $A\leftarrow E$ be an assignment, where $A$ is a left-hand-side variable, and $E$ is a right-hand-side expression to assign. Given the executed result $e$ of $E$ and the guessed representations of $A$ and $K-1$ randomly selected left-hand-side variables $\mathbf {LHS} = \{A,A_1,...A_{K-1}\}$, NI learns to classify $A$ in $\mathbf {LHS}$ through cross-entropy.

\textit{Argument discrimination loss $L_2$.} NI is trained to recognize a function call with randomly replaced arguments, as NSXI only holds for valid function calls. For example, the returned abstract semantics of \texttt{celsius\_to\_fahrenheit(duck)} is unclear. The return value from a real function call is labeled True, and the return value from the same call with an argument replaced by a randomly sampled object is labeled False. Binary cross-entropy (BCE) is used.

\textit{Data-flow discrimination loss $L_3$.} NI is trained to recognize data-flow paths \cite{liu1998simple,ferrante1987program}, which can differentiate two variables of the same name that language modeling objectives cannot (Appx. \ref{appx:ground}). For example, \texttt{fahrenheit} can likely be computed from variables \texttt{celsius}, \texttt{temperature}, or \texttt{thermometer}, but not \texttt{passport} or \texttt{duck}. 
A pair of objects with a data-flow path is labeled True, and a pair of randomly sampled objects without data-flow is labeled False. BCE is used.

\begin{wraptable}[9]{R}{0.5\textwidth}
\vspace{-15pt}
    \centering
\caption{Accuracy for Neural Interpretation learning the three abstract semantics objectives.}
\vskip -0.5em
\label{tab:three}
\resizebox{\linewidth}{!}{
\begin{tabular}{cclll}
\toprule
\multirow{2}{*}{Fine-tuning} & \multirow{2}{*}{Abstract semantics} & \multicolumn{3}{c}{Accuracy} \\
                             &                                     & Train    &  Valid.   &   Test   \\ \midrule
\multirow{3}{*}{CodeBERT}   &  Return variable classification      & 95.63    &  96.25   &   95.57  \\
                            &    Argument discrimination           & 90.57    &  91.24   &   90.75 \\
                            &    Data-flow discrimination          & 94.56    &  94.59   &   94.67 \\ \midrule
\multirow{3}{*}{None}   &  Return variable classification          & 86.37    &  86.39   &  85.89 \\
                            &    Argument discrimination           & 79.66    &  79.85   &   79.79 \\
                            &    Data-flow discrimination          & 91.67    &  91.54   &   91.78 \\ \bottomrule
\end{tabular} 
}
\end{wraptable}

\paragraph{NI performance.} We compare training NI with or without fine-tuning. We first train NI by fine-tuning both Guesser and Executor from CodeBERT \cite{feng2020codebert} large language model pre-trained on code corpus. We then train NI from randomly initialized weights of the same network. For both settings, the three abstract semantics objectives are trained jointly on filtered ETH Py150 Open dataset for 3 epochs, completed within 7 hours on a consumer GPU. Batch size is 16. Fig. \ref{fig:three} plots NI accuracy during training. The final training, validation, and testing accuracy are presented in Tab. \ref{tab:three}.

We see from Fig. \ref{fig:three} that Neural Interpretation can learn to model program execution through back-propagation. As all three objectives are constructed with balanced datasets, Tab.~\ref{tab:three} shows non-trivial validation and testing accuracy on all three abstract semantics definitions, proving that learning of Neuro-Symbolic Execution Problem is feasible with our Neural Interpretation model, and the learning objectives can be flexibly defined. Aligned with the standard conclusion from LLM literature \cite{radford2018improving, brown2020language}, we find that fine-tuning improves NI performance. 
Learning NSXP is possible. NI shows that we can factorize a source code into a computational graph and model each computation individually, a distinctive ability that standard LLMs do not have with sequence models of source code.

\subsection{White-Box Variable Misuse Localization and Repair}
\label{sec:misuse}

We show that by executing a buggy source code, NI can localize and repair variable misuses.

\paragraph{Dataset and task.} 
CuBERT \cite{kanade2020learning} proposes ``variable misuse localization and repair'' dataset containing single Python function definitions, where a variable at some locations is replaced by another variable that is defined within the function as a misuse. 
Following \cite{kanade2020learning, vasic2019neural}, we propose to use this dataset as a four-step \textit{white-box variable misuse localization and repair} task designed to evaluate Neuro-Symbolic Execution models. Given a function with a variable misuse, the model needs to perform four steps. 1) \textit{Code classification:} Given the code, classify whether a misuse exists. 2) \textit{Call localization:} Given the return vectors of all function calls, localize the call with a misused argument. 3) \textit{Argument localization:} Given a call with a misused argument, localize the misuse. 4) \textit{Repair:} Select the correct variable given the return value of the localized call by substituting the misused argument with each variable in scope. 
This task evaluates white-box execution with NSX models as all four steps can be performed by executing individual function calls. The NSX model needs to learn to return abstract semantics vectors that distinguish authentic and misused arguments.

\paragraph{White-box variable misuse localization and repair with Neural Interpretation.}
For \textit{code classification}, a BERT-style attention layer takes the return values of all function calls and produces a binary classification whether a misuse exists from the CLS output \cite{devlin2018bert}.

\begin{wraptable}[8]{R}{0.38\textwidth}
\vspace{-10pt}
\caption{NI accuracy for white-box variable misuse localization and repair.}
\vskip -0.5em
\label{tab:misuse}
\resizebox{\linewidth}{!}{
\begin{tabular}{cccc}
\toprule
Task Step             & \multicolumn{3}{c}{Accuracy} \\
                      & Train    & Valid.   & Test   \\
                      \midrule
Code Classification   & 94.46    & 93.45    & 93.53   \\
Call Localization     & 94.10    & 92.97    & 93.13  \\
Argument Localization & 98.90    & 97.86    & 97.83  \\
Repair                & 94.02    & 92.81    & 92.97  \\
\bottomrule
\end{tabular}
}
\end{wraptable}

For \textit{call localization}, we define a contaminated function call with a misuse variable as an argument, or one with a contaminated return value as an argument, recursively. 
For training only, a one-layer neural network learns to classify whether a function call is contaminated given its return value. For both training and inference, an attention layer takes the return values from all function calls and classifies the first contaminated call with the original misused argument.

For \textit{argument localization}, a one-layer neural network takes the Transformer output vector $a'$ for each argument $a$ to classify the misused argument. Note that in all other tasks, $a'$ is discarded, as only the return vector $r$ is kept as the execution result (Eq. \ref{eq:return}). 

For \textit{repair}, like a debugger, we pause when the misused call is executed to collect all variables $\mathbf o$ in the Interpreter memory, among which the correct argument exists. We substitute the misused argument with each variable $o\in \mathbf o$ and compute the $\lambda$ call to obtain the return values, which are given to a one-layer network to classify the correct argument.
Appx. \ref{appx:loss} defines the loss functions in detail.

\begin{wraptable}[7]{R}{0.45\textwidth}
\vspace{-14pt}
\caption{NI performance compared to LSTM and CuBERT baselines.}
\vskip -0.5em
\label{tab:baseline}
\resizebox{\linewidth}{!}{
\begin{tabular}{cccccc}
\toprule
Model    & Epochs & \multirow{2}{1.7cm}{\centering Classification Accuracy} & \multirow{2}{1.5cm}{\centering Localization Accuracy} & \multirow{2}{1.5cm}{\centering Loc+Repair Accuracy} \\\\
\midrule
LSTM                   & 100        & 79.30                   & 64.39                 & 56.89               \\
CuBERT                 & 2          & 94.87                   & 91.14                 & 89.41               \\
\midrule
\textbf{NI}                   & 2           & 94.33           & 93.47                 & \textbf{90.43}               \\
\bottomrule
\end{tabular}
}
\end{wraptable}
\paragraph{NI performance.}
Neural Interpretation is initialized with CodeBERT \cite{feng2020codebert} Guesser and CodeGPT \cite{lu2021codexglue} Executor, and it is fine-tuned for 2 epochs (60 hours). Batch size is 64. Filtering and truncation are applied. We exclude around 0.7\% scripts with Code Generator errors, and we exclude 0.5\%
programs whose misuse variable is not a function argument and thus cannot be detected by execution (filtering details in Appx. \ref{appx:dataset}).
The performance is presented in Tab. \ref{tab:misuse}, and comparison with baseline models are in Tab. \ref{tab:baseline}. %
NI shows competitive ability to solve variable misuse localization and repair. By executing a buggy source code, even when concrete inputs are missing, a variable misuse can be identified and fixed. NI's overall localization and repair accuracy is higher than that of the baseline CuBERT model, with significantly higher localization accuracy alone. Note that NI is fine-tuned from models trained for language modeling objectives, and if pre-trained for execution-based objectives, NI's fine-tuning performance may further improve in future work.
Lastly, the white-box execution emulates a program debugger, representing variables and their abstract semantics values at every step, which can explain a predicted misuse by tracing the source function call. In contrast, standard LLMs do not model execution and cannot explain their predictions by default, as sequence models of code discard the program structure.

\section{Conclusions and Discussions}

Neural Interpretation is a new model of computers that models step-by-step execution of source code. It is the first to achieve the following five desiderata at the same time: 1) neural representation learning from name semantics; 2) white-box step-by-step execution; 3) execution of partial source code  without concrete inputs; 4) execution of general-domain source code; 5) composition of neural networks ``programmed'' according to the source code.

\bibliography{my}
\bibliographystyle{plain}

\newpage
\appendix

{\large Supplementary Material to \textit{Neuro-Symbolic Execution of Generic Source Code}}

\section{A road map of Neuro-Symbolic Execution sub-problems}
\label{appx:roadmap}

We list the sub-problems of the Neuro-Symbolic Execution Problem that are addressed by the Neural Interpretation Model.
\begin{itemize}[topsep=-5pt,itemsep=-1ex,partopsep=1ex,parsep=2ex, leftmargin=2ex]
    \item \texttt{SP-\ref{sp:new_language}}. \textit{Language support of Neural Interpretation.} We have selected Python to be the language supported by NI, without loss of generality. We choose Python because the design of the language emphasizes semantics with dynamic typing and duck typing, which are difficult to analyze with symbolic models. In practice, Python's language design encourages developers to rely on semantically meaningful names to express and document program logic, rather than formal channels such as static type checking. Python's language design aligns with our goal to model source code semantics. Also, Python is the second most popular language for open-source repositories (GitHub statistics). Other languages can be supported (see Appx. \ref{appx:language}). 
    \item \texttt{SP-\ref{sp:loop}}.\textit{ How does Neural Interpretation handle loops?} We address this problem with the linearity principle. Alternatively, an execution model may predict the exact number of times a loop will be executed. This requires an assumption that concrete input is given and achieves a ``Concrete Neural Execution'', which is a special case of the Neuro-Symbolic Execution Problem studied in all existing neural execution work, to the best of our knowledge.
    \item \texttt{SP-\ref{sp:gran}}. \textit{How does Neural Interpretation represent lists?} We address this problem with the named representation principle. Alternatively, an execution model may attempt to represent every element in a list (dictionary, set, tuple ...) with a referenced vector. The advantage is that the representation is more detailed. The disadvantage is that this may not be possible without context, and it may consume too much memory.
    \item \texttt{SP-\ref{sp:nis}}. \textit{NIS design.} We design the NIS as an interface for the Neural Interpreter to invoke neural representations \ref{tab:nis}. Our NIS is general and can be used for other languages. However, our NIS may be most suitable for object-oriented programming languages like Python. 
    \item \texttt{SP-\ref{sp:infinite_fns}}. \textit{How do we represent a Python function with a unique neural network, especially as the number of Python functions is infinite?} We think that predicting neural networks with a neural network is a general problem with applications other than Neuro-Symbolic Execution Problem. We solve this by currying (Eq. \ref{eq:currying}): a single neural network takes a vector function abstract semantics as a partial input. Infinite number of possible vectors that are predicted on-demand from source code functions, which avoids potentially infinite parameters needed. 
    \item \texttt{SP-\ref{sp:missing_fn_def}}. \textit{What happens if the function definition is missing from source code?} A function call may not have its definition in generic source code, and, vice versa, a definition may not have a concrete call site. The function definition is inferred from its name, emulating how humans read the source code. 
    \item \texttt{SP-\ref{sp:have_fn_def}}. \textit{How does NI represent functions if their definitions are available in the source code? How should recursion be handled?} We compile it once into a vector signature to follow the linearity principle. For recursion, we use guessed representations and do not step into the recursive calls within the definition. Alternative designs may choose to inline expand the function definition \cite{serrano1997inline}, which may be a more detailed neural representation of function definitions.
    \item \texttt{SP-\ref{sp:ground_truth}}. \textit{Ground truth problem.} Abstract semantics may be defined flexibly, similar to how vector encodings are flexibly used as a general representation in other deep learning methods. We do not study pre-training of Neural Interpretation in this paper. If we do pre-train NI, we would need to find the ``best'' abstract semantics --``a gold standard''-- which has general information that can transfer to downstream tasks universally. However, as we argued in Appx. \ref{appx:ground}, language modeling objectives are not the gold standard of abstract semantics. We consider it an interesting research direction. 
    \item \subproblem{sp:pooling}\texttt{SP-\ref{sp:pooling}}. \textit{The Guesser's pooling methods.} Which set of tokens should be pooled to guess the representation of an expression? The goal of pooling is give as much information as possible to the Executor, while disentangling the expression from the rest of the context. Our pooling methods are given in Appx. \ref{appx:engineering}.
\end{itemize}

In the Limitations Section of Appx. \ref{appx:limitation}, we also present some additional research problems that are not addressed in this paper and may be tackled by future Neuro-Symbolic Execution models. 

\section{Language modeling is not enough to solve the ground truth problem}
\label{appx:ground}

Neural Interpretation may be trained with language modeling objective (e.g. MLM \cite{devlin2018bert}). However, the example in Fig.~\ref{fig:assign_twice} shows its limitation. Variable \texttt{b} is assigned twice at two locations $B_1, B_2$ with different abstract semantics $b_1, b_2$. If a variable name is the gold standard for the abstract semantics, then $b_1=b_2$. However, NI should distinguish $b_1, b_2$ such that $\texttt{less\_than}(b_1,b_2)$ is True and $\texttt{less\_than}(b_2,b_1)$ is False. After all, humans can deduce this fact with mental execution. 

\begin{figure}[t]
    \centering
\begin{lstlisting}[style=mypython]
# Python
def assign_twice(a):
    b = a+1
    b = b+1
    return b
\end{lstlisting} %
    \caption{As a variable can be assigned more than once, its name alone is not a sufficient ground truth for abstract semantics, and language modeling objectives are not enough.}
    \label{fig:assign_twice}
\end{figure}

\section{Architecture engineering details}
\label{appx:engineering}

The architecture of Neural interpretation follows compiler design (Fig. \ref{fig:archi}). NI has the following components: Parser, Code Generator, Neural Interpreter, Guesser, and Executor. Neural Interpretation is a Python interpreter architecture written in Python. Our source code is publicly available for exploration. All datasets are open-sourced with public usage licenses to reproduce the study.

\paragraph{Parser.} We use an existing open-source parser called Tree-sitter \cite{brunsfeld} that currently supports 46 languages. Given a source code, the parser returns its abstract syntax tree (AST). Tree-sitter has a cursor interface that returns an iterator to walk through all nodes on the AST. We invoke the code generator on every AST node iterated by the cursor.

\paragraph{Code Generator.} We designed the Code Generator for Neural Interpretation. The Code Generator traverses an abstract syntax tree and invokes the Neural Interpreter to execute each AST node (Python statement) one-by-one and perform its language construct (e.g. while-loops, list comprehensions, function calls). Code Generator is a standard component in compilers, often implemented with visitor pattern to achieve double dispatch, as seen in our Code Generator as well. Code Generation alone takes around 10 milliseconds per script on average on a single CPU core. The Code Generator can parse a majority of Python scripts in the wild (Tab. \ref{tab:dataset}). We provide a few examples to illustrate code generation implementation in Fig. \ref{fig:codegen_assign} to Fig. \ref{fig:codegen_import}.

\begin{figure}[h]
\begin{lstlisting}[style=mypython]
# Python
def visit_assignment(self, ast_node: AcceptAssignment):
    children = ast_node.children
    assert children[1].type == "="
    lhs = children[0]
    rhs = children[2]
    val_visitor = VisitorValue(self.inter)
    # val is the returned vector for the rhs expression
    val = make_accepter(rhs).accept(val_visitor)
    self.interpreter.store(lhs, val)
    return val
\end{lstlisting} %
\caption{Code generation for assignment statements.}
\label{fig:codegen_assign}
\end{figure}

\begin{figure}[h]
\begin{lstlisting}[style=mypython]
# Python
def visit_identifier(self, ast_node: AcceptIdentifier):
    var = self.interpreter.lookup(ast_node)
    # for variable misuse benchmark only
    if ast_node.source_contamination:
        var.contaminated = True
    return var
\end{lstlisting} %
\caption{Code generation for identifiers. Source contamination flag is set only for the variable misuse task.}
\label{fig:codegen_identifier}
\end{figure}

\begin{figure}[h]
\begin{lstlisting}[style=mypython]
# Python
def visit_list(self, ast_node: AcceptList):
    children = ast_node.children
    assert children[0].type == "["
    assert children[-1].type == "]"
    vals = []
    inter = self.interpreter
    for i in range(1, len(children) - 1):
        vis = VisitorValue(inter)
        # get the vector representation of every element
        val = make_accepter(children[i]).accept(vis)
        vals.append((val, children[i]))
    a_list = inter.get("__list_of__", language_construct=True)
    the_list = inter.lambdaa(ast_node, a_list, *vals)
    return the_list
\end{lstlisting} %
\caption{Code generation for lists, reflecting the named representation principle.}
\label{fig:codegen_list}
\end{figure}

\begin{figure}[h]
\begin{lstlisting}[style=mypython]
# Python
def visit_while_statement(self, ast_node: AcceptWhileStatement):
    children = ast_node.children
    val_visitor = VisitorValue(self.interpreter)
    # children[1] is the iterator.
    # vector for the condition of while loop
    condition = make_accepter(children[1]).accept(val_visitor)
    inter = self.interpreter
    fn = inter.lookup("__while__", language_construct=True)
    # the context vector for the while-condition
    context = inter.lambda_call(ast_node, fn, condition)
    # within this with-block, all \texttt{lambda} calls will receive a context vector representing the control flow
    with inter.control_flow_context(ast_node, context):
        # performs code generation on the while body
        make_accepter(cs[3]).accept(self)
\end{lstlisting} %
\caption{Code generation for while loops, reflecting the linearity principle.}
\label{fig:codegen_while}
\end{figure}

\begin{figure}[h]
\begin{lstlisting}[style=mypython]
# Python
def visit_import_from_statement(self, ast_node: AcceptImportFromStatement):
    pass
\end{lstlisting} %
\caption{Code generation for import statements, which has no effect. Not all language constructs in Python are implemented to simplify engineering.}
\label{fig:codegen_import}
\end{figure}

\newpage

\paragraph{Neural Interpreter.} The Neural Interpreter emulates the Python interpreter that executes program statements. It is a virtual machine that provides the runtime environment.  The Interpreter holds a stack of available scopes per Python scoping rules that the current execution line has access to. Every scope has a memory table, implemented as a hash table with string keys and vector values. When executing the source code, at any point in time, the memory tables contain all previously defined variables in the scope and their current values. When a variable is assigned a new vector representation, the memory table is updated accordingly. The Interpreter holds a stack of control flow contexts, representing the nested control flow conditions that the current execution line belongs to.

To reflect object-oriented nature of Python, Neural Interpreter define additional classes such as \texttt{Variable} and \texttt{Object}, such that every \texttt{Variable} can keep track of every \texttt{Object} that it was assigned to in the past by static-single assignments \texttt{SSA}. The \texttt{Scope} class implements memory tables and provide with-statement interface to the Code Generator for easy scope creation. Every \texttt{Object} keeps track of the vector value as well as the AST node that computes it, which is only used for diagnosis and code analysis (e.g. to label source contamination in variable misuse), not by the Executor. \texttt{ExecutableScript} is a subclass of \texttt{Object} that holds the compiled vector signature of a function definition. All of these classes are just wrappers for abstract semantics vectors: only the abstract semantics vectors are used by $\lambda$-Executor.

To analyze the execution of a script and compute various loss functions, the Interpreter has various ``hook'' functions that can be triggered when execution events happen. For example, when $\lambda$ is called, we register the $\lambda$ call and store its control flow contexts, function signature for currying, and arguments, so that we can run this $\lambda$ call again if needed or exchange the arguments for the argument discrimination task, for example.

\paragraph{Guesser.} 
The Guesser uses a BERT Transformer \cite{devlin2018bert}. The Guesser $\gamma$ is formally defined as  $\mathcal \gamma(S) = \mathbf G$ per Eq. \ref{eq:guesser}. To recap, given a source code string $S$ that contains expressions $E_1, E_2, ... E_N$ (identifiers included), the Guesser $\gamma$ outputs a matrix $\mathbf G: \mathbb R^{N\times H}$, where $H$ is  
For a tokenized source code sequence $[t_1, t_2,... t_L]$,
\begin{align}
    \texttt{BERT}_\gamma([t_1, t_2,... t_L]) &= [a_1, a_2, ... a_L] \label{eq:guessbert}
\end{align}
For an expression (identifier included) $E$ that corresponds to $K$ token indices $t_{E,1}, t_{E,2}... t_{E,K}$, its guessed representation is given by 
\begin{align}
    g_E = \textit{maxPool}(a_{E,1}, a_{E,2}... a_{E,K}), \label{eq:pool}
\end{align}
where \textit{maxPool} is takes the maximum over $K$ vectors for every vector index $1, 2...H$ to return another $H$-element vector $g_E$. The vector $g_E$ is a row in the Guesser $\gamma$'s output matrix $\mathbf G$. A type embedding is added to the pooled result $g_E$ based on the AST node type of expression $E$. The set of tokens corresponding to each expression is discussed below.

We use Huggingface libraries \cite{wolf2019huggingface}, which provide CodeBERT \cite{feng2020codebert} for fine-tuning. The Guesser Transformer is a standard BERT Transformer model \cite{devlin2018bert} that receives a token sequence and returns one vector for each token. Some built-in functions are looked up from an embedding matrix rather than guessed from the source code. The built-in functions are listed in Appx. \ref{appx:build_in}.

\paragraph{Guesser Pooling.}  For each token sequence, the Guesser Transformer outputs a list of vectors for each token \cite{devlin2018bert}. A max pool over the corresponding tokens for each expression (identifier included) produce a single vector $v$, which is then stored on the AST node. The set of corresponding tokens is chosen by the Code Generator and depends on the Python language construct, designed to select all relevant tokens to maximize performance while excluding irrelevant tokens. For example:
\begin{enumerate}[topsep=-5pt,itemsep=-1ex,partopsep=1ex,parsep=1.3ex, leftmargin=3ex]
    \item For constant values, such as float \texttt{3.1415926}, string \texttt{"lorem\_ipsum"}, the pooled tokens simply correspond to the source code substring.
    \item For an expression such as \texttt{1+1}, tokens in the expression \texttt{1+1} are pooled.
    \item For a variable or function without definitions, such as \texttt{duck}, tokens in the name are pooled.
    \item For an assignment statement such as \texttt{a = 1+1}, the object on the right-hand-side expression is directly stored in the Interpreter memory by the identifier. Therefore, the guessed representation of variable \texttt{a} is the guessed representation of \texttt{1+1}. The semantics of the name ``a'' is not lost, because when \texttt{a} is used in the next expression, it will be a part of the guessed semantics. For this reason, the semantics of both left-hand-side and right-hand-side are preserved, which maximally represents program semantics at the guessing stage.
    \item For a function whose definition is available, tokens in the definition body are pooled, since every token in the definition body can potentially determine the function's semantics. The function name refers to this guessed representation due to assignment behavior (Case 4). 
    \item For an iterator variable in a for-statement, such as \texttt{for i in range(10)}, the iterable \texttt{range(10)} stepped through is pooled and becomes the guessed representation of \texttt{i} due to assignment behavior (Case 4).
\end{enumerate}
Other pooling methods may be designed (\texttt{SP-\ref{sp:pooling}}).

Instead of pooling right after the Guesser Transformer's output for all possible token subsets, we perform pooling on-demand only when an Executor requires a variable's initial values, which avoids pooling unneeded expressions. If the pooled tokens cannot be found, we use a learnable vector as the default value (see Appx. \ref{appx:build_in}).

\paragraph{$\lambda$-Executor.} The Executor is formally defined as $ f(a) = \lambda(\theta_f, a)=r$ by Eq. \ref{eq:return}. 

We use Huggingface libraries \cite{wolf2019huggingface}. The Executor can fine-tune from CodeBERT \cite{feng2020codebert} or CodeGPT \cite{lu2021codexglue}. 
If fine-tuning from CodeBERT, the Executor Transformer has BERT architecture with vector input sequence $[\theta_f', c_1, c_2, ... c_M, a_0, a_1, ... a_N]$. 
\begin{align}
    \texttt{BERT}_\lambda([\theta_f', c_1, ... c_M, a_0, ... a_N]) &= [r, c'_1, ... c'_M, a'_0,... a'_N ] \label{eq:execbert}\\
    \lambda([\theta_f', c_1, ... c_M, a_0, ... a_N]) &= r  \label{eq:execret}
\end{align}

If fine-tuning from CodeGPT, the Executor Transformer has GPT architecture with vector input sequence $[\theta_f', c_1, c_2, ... c_M, a_0, a_1, ... a_N, \texttt{EOS}]$, where \texttt{EOS} is a token that marks the end of the sequence. 
\begin{align}
    \texttt{GPT}_\lambda([\theta_f', c_1, ... c_M, a_0, ... a_N, \texttt{EOS}]) &= [r, c'_1, ... c'_M, a'_0,... a'_N, o] \label{eq:execbert}\\
    \lambda([\theta_f', c_1, ... c_M, a_0, ... a_N, \texttt{EOS}]) &= o  \label{eq:execret}
\end{align}

The vector $\theta_f'$ is a single vector mapped from multi-vector signature parameter $\theta_f$ of the called function $F$ for currying. The list of vectors $c_1, c_2, ... c_M$ are the control flow context vectors. The list of vectors $a_0, a_1, ... a_N$ are the arguments. Additionally, three type embeddings are added to control flow context vectors, function signature, and arguments respectively to differentiate them. Both the guessed and executed vector representations for each argument are concatenated and given to the Executor as inputs, in order to improve execution modeling with more information. The return value $r$ is the first vector of the BERT output aligned with the $\theta_f'$ input vector, similar to the standard classification practice with the CLS token \cite{devlin2018bert}. The return value $o$ is the output for GPT, aligned with the input EOS token. When multiple return values are needed, we may unpack the single return value $r$ to $k$ return values $r_1, r_2, ... r_k$ by running $\lambda$-Executor with the built-in \texttt{\_\_unpack\_k\_\_} function (Appx. \ref{appx:build_in}).

\paragraph{Neural Compilation.} Function signature parameter $\theta_f$ is obtained through guessing or neural compilation if the function definition is available. To simplify the implementation, instead of using a tuple signature $\theta_f=(p,r)$, we map the signature to a single vector $\theta_f'$ during neural compilation, in order for all Neural Interpretation objects to be homogeneously represented by a single vector with a constant dimension $H$.

Function signature parameter can take into accounts where the function $F$ may update its parameters during execution. We use this more general form in our implementation. The generalization does not impose additional assumptions on our theory: updated parameters are just additional return values. To neurally compile the definition into $\theta_f'$, we use vector tuple $(g_f, p_1, p_2, ...p_U, r, p_1', p_2', ... p_U')$ as the signature instead of $(p, r)$. Vector $g_f$ is the guessed representation of the function to improve neural compilation modeling quality. The list of vectors $p_1, p_2, ... p_U$ are the parameters before the function $F$ is called. The list of vectors $p_1', p_2', ... p_U'$ are the parameters after the function $F$ is called. Vector $r$ is the return value. The neural network that maps the signature into a single vector is simply the $\lambda$-Executor currying a built-in function \texttt{\_\_compile\_function\_\_} (see Appx. \ref{appx:build_in} for all built-in functions of Neural Interpretation).

Fig. \ref{fig:codegen_compilation} presents the Code Generator logic for neural compilation.

\begin{figure}[h]
\begin{lstlisting}[style=mypython]
# Python
def visit_function_definition(self, ast_node: AcceptFunctionDefinition):
    children = ast_node.children
    inter = self.interpreter
    with inter.new_scope("func: " + fun_name):
        parameter_visitor = VisitorParameter(inter)
        parameters = make_accepter(children[2]).accept(parameter_visitor)
        # the parameters need to exist in the scope
        before_parameters = [p for p in parameters]

        block_visitor = VisitorBlock(inter)
        make_accepter(children[4]).accept(block_visitor)

        try:
            ret = inter["__return_val__"]
        except KeyError:
            ret = inter.executor.none_embedding

        # get side effects on all parameters
        after_parameters = parameters

    guessed_fun_repr = inter.guess_value(ast_node.node)
    compile_function = inter.get("__compile_function__", language_construct=True)
    signature = (guessed_fun_repr, *before_parameters, ret, *after_parameters)
    compiled_signature = inter.lambdaa(ast_node.node, compile_function, *signature)
    new_function_obj = inter.new_function(ast_node.node, compiled_signature)
    # assign the function to its name
    inter.assign(acc.node.children[1], new_function_obj)
\end{lstlisting} %
\caption{Code generation for neural compilation of function definitions.} %
\label{fig:codegen_compilation}
\end{figure}

\paragraph{Pooled multi-threading} Neural networks are typically trained with batched inputs to leverage parallel computation on GPUs. The Guesser can be executed in batch by collating input sequences and splitting them after. To achieve batched Executor, every script has an Executor worker thread that runs the Code Generator and the Neural Interpreter. The Executor worker has the same interface as a normal Executor, but instead of executing the \texttt{lambda} call right away, it synchronizes for all live Executor workers in the batch. When every live worker in the batch has received their \texttt{lambda} call, all the worker calls are collated into a batched tensor input and executed on the Executor Transformer. The batched \texttt{lambda} call result is then split and sent back to each Executor worker. When the Executor worker receives its \texttt{lambda} call result, it resumes Code Generator and Neural Interpreter.

The Executor workers are pooled with a factor of 2 for an even allocation of workload per thread and speed up execution. For example, if a batch input has 16 scripts, we create 8 Executor workers. During batched execution, whichever worker finishes first will receive another script to be executed. A long script, therefore, may occupy only one worker, while another worker execute multiple shorter scripts during the same time.

\section{Additional language support}
\label{appx:language}
Adding new languages other than Python to Neural Interpretation may require additional implementation to extend the Code Generator for the new language constructs. For each language, a new Code Generator needs to be written to specifically handle every language construct by its grammar specification.

The Tree-sitter parser \cite{brunsfeld} supports 46 languages, and open sourced parser likely exist for additional Neural Interpretation languages.

The Neural Interpreter and its Neural Instruction Set are language-agnostic and may be re-used directly for different languages, which is similar to how the same CPU on normal computers can execute programs in different languages. 

The Guesser may need to be pre-trained on the new language. 
The Executor models semantics and not syntax, so its architecture and even parameters may be reused directly for different languages.

\section{Neural Interpretation limitations}
\label{appx:limitation}

Some sub-problems of the Neuro-Symbolic Execution Problems are encountered but not addressed. Below is a list of limitations of Neural Interpretation as the first Neuro-Symbolic Execution model that point to future work directions.
\begin{itemize}[topsep=-5pt,itemsep=-1ex,partopsep=1ex,parsep=1ex, leftmargin=2ex]
    \item \subproblem{sp:alias}\texttt{SP-\ref{sp:alias}}. \textit{Pointer aliasing problem.} How should an Neuro-Symbolic Execution model handle pointer aliasing? Aliasing is needed to model object-oriented programming features in detail, and it may require a more advanced external memory than our string-vector hash map.
     \item \subproblem{sp:class}\texttt{SP-\ref{sp:class}}. \textit{Aliasing for detailed object representation.} How should classes, instances, attributes be represented by an Neuro-Symbolic Execution model? Recall that in Neural Interpretation, every object is represented by a vector. Therefore, attributes are evaluated on demand when accessed, from a built-in $\texttt{\_\_get\_attr\_\_}$ signature (see Appendix \ref{appx:build_in}) and the attribute name. Class definitions are Python programs, and instantiation uses a vector signature instead of Python's init method. Neural Interpretation handles OOP in a way that is consistent with the string-vector external memory representation and is robust against missingness of source code. Even when class definition of an object is missing, the meaning of an object attribute can still be obtained through guessing, and class instantiation (e.g. \texttt{Apple()}) can be performed given a single vector guessed from the class name (\texttt{Apple}). However, representing detailed attributes of an object can be preferable if the class definition is available, which would require additional methods not introduced in this paper, possibly an improvement of the string-vector external memory.
    \item \subproblem{sp:side_effect}\texttt{SP-\ref{sp:side_effect}}. \textit{Aliasing for side effects.} Can side effects of function calls be modeled by an Neuro-Symbolic Execution model? In Neural Interpretation, every function definition is compiled into a vector signature. When a compiled function is called, the vector signature does not have information regarding how some arguments may be updated as a side effect. Updating all arguments is an obvious alternative to model side effects, but it may not be effective. Side effects are often not encapsulated in a statement: due to aliasing, the execution of a function may potentially update all objects in the memory and not just its arguments. If every function updates every object in the memory, then every function call is just a global update, similar to that of a global attention mechanism.  What arguments are used in a function call is important and expresses encapsulation information, and treating every object in the scope as an argument of every function call indiscriminately would forgo the program structure that forms the inductive bias of Neural Interpretation.
    \item \subproblem{sp:curryhoward}\texttt{SP-\ref{sp:curryhoward}}. \textit{Connection with Curry-Howard correspondence.} Is it possible to formally connect abstract semantics of NSXP to classical theoretical computer science results of Curry-Howard correspondence \cite{howard1980formulae}?
    \item \subproblem{sp:engineering}\texttt{SP-\ref{sp:engineering}} \textit{Complete Python language support.} Code Generator supports most but not all Python constructs to simplify engineering. Since new language features are introduced frequently in new Python releases, constant development and maintenance of code generator is needed to achieve full language support.
\end{itemize}

\paragraph{Pointer aliasing.} 
Limitations \texttt{SP-\ref{sp:alias}}, \texttt{SP-\ref{sp:class}}, \texttt{SP-\ref{sp:side_effect}} of Neural Interpretation are due to the pointer aliasing problem.  Neural Interpretation uses static source code as input. It is impossible to compute statically precise alias information, including may-alias and must-alias, in generic source code that we are interested in \cite{ramalingam1994undecidability}. Many issues related to pointer aliasing remain open research problems \cite{hind2001pointer, vedurada2019batch}. It is even more difficult for dynamic languages like Python \cite{gorbovitski2010alias} and impossible for partial source code that NI learns from. Neural Interpretation does not handle aliasing in order to simplify modeling assumptions and achieve generality. Instead, Neural Interpretation achieves Neuro-Symbolic Execution with the design decision such as vector-based object representation and procedural programming style source code execution modeling with no side effects, as we discussed previously, which avoids the need for precise pointer alias information that cannot be obtained.  %

\section{Dataset and truncation details}
\label{appx:dataset}

We use two datasets, the ETH Py150 Open dataset and the variable misuse localization and repair dataset, both introduced from CuBERT \cite{kanade2020learning}. During pre-processing, we filter scripts that cause an error in Code Generator or are longer than 10000 characters that are the long tail of the dataset (see figures below). The number of original data points and filtered data points are presented in Tab. \ref{tab:dataset}. 

\begin{table}[t]
\caption{Filtering size for the ETH Py150 Open dataset and white-box variable misuse localization and repair datasets \cite{kanade2020learning}, with various filtering conditions including the length of the code, code generation errors, variable misuse labeling errors.}
\label{tab:dataset}
\centering
\begin{tabular}{rrrrrrr}
\toprule
                      & \multicolumn{3}{c}{Py150}       & \multicolumn{3}{c}{Variable Misuse \& Repair} \\ 
                      & Training & Validation & Test    & Training      & Validation      & Test        \\
                      \midrule
Original              & 74749    & 8302       & 41457   & 700708        & 75478           & 378440      \\
Too Long              & 9971     & 1092       & 5531    & 0             & 0               & 0           \\
Code Generation Error & 1346     & 114        & 588     & 5044          & 456             & 2729        \\
Misuse Labeling Error & n/a      & n/a        & n/a     & 3680          & 352             & 1890        \\
Filtered              & 63432    & 7096       & 35338   & 691984        & 74670           & 373821      \\
Filtered Percentage   & 84.86\%  & 85.47\%    & 85.24\% & 98.75\%       & 98.93\%         & 98.78\%    \\
\bottomrule
\end{tabular}
\end{table}

We present the histogram (100 bins) for the number of characters and the number of $\texttt{lambda}$ calls per script in ETH Py150 Open dataset and variable misuse localization and repair datasets \cite{kanade2020learning} from Fig. \ref{fig:py150len} to Fig. \ref{fig:misusecall} (training, validation, test set combined).

\begin{figure}
    \centering
     \begin{subfigure}[b]{0.48\linewidth}
         \centering
         \includegraphics[width=\linewidth]{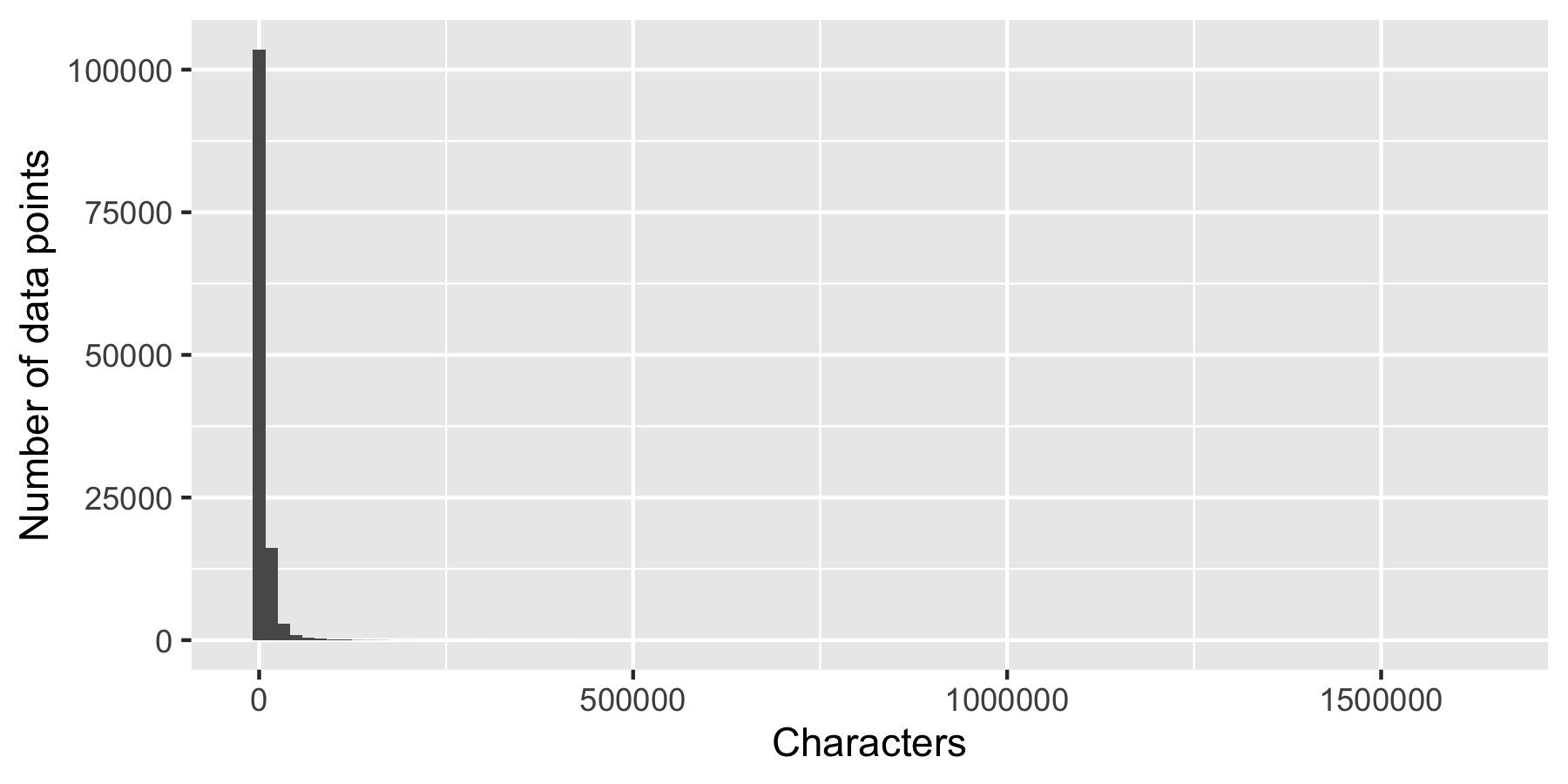}
         \caption{The entire ETH Py150 Open dataset. }
     \end{subfigure}\hfill
      \begin{subfigure}[b]{0.48\linewidth}
         \centering
         \includegraphics[width=\linewidth]{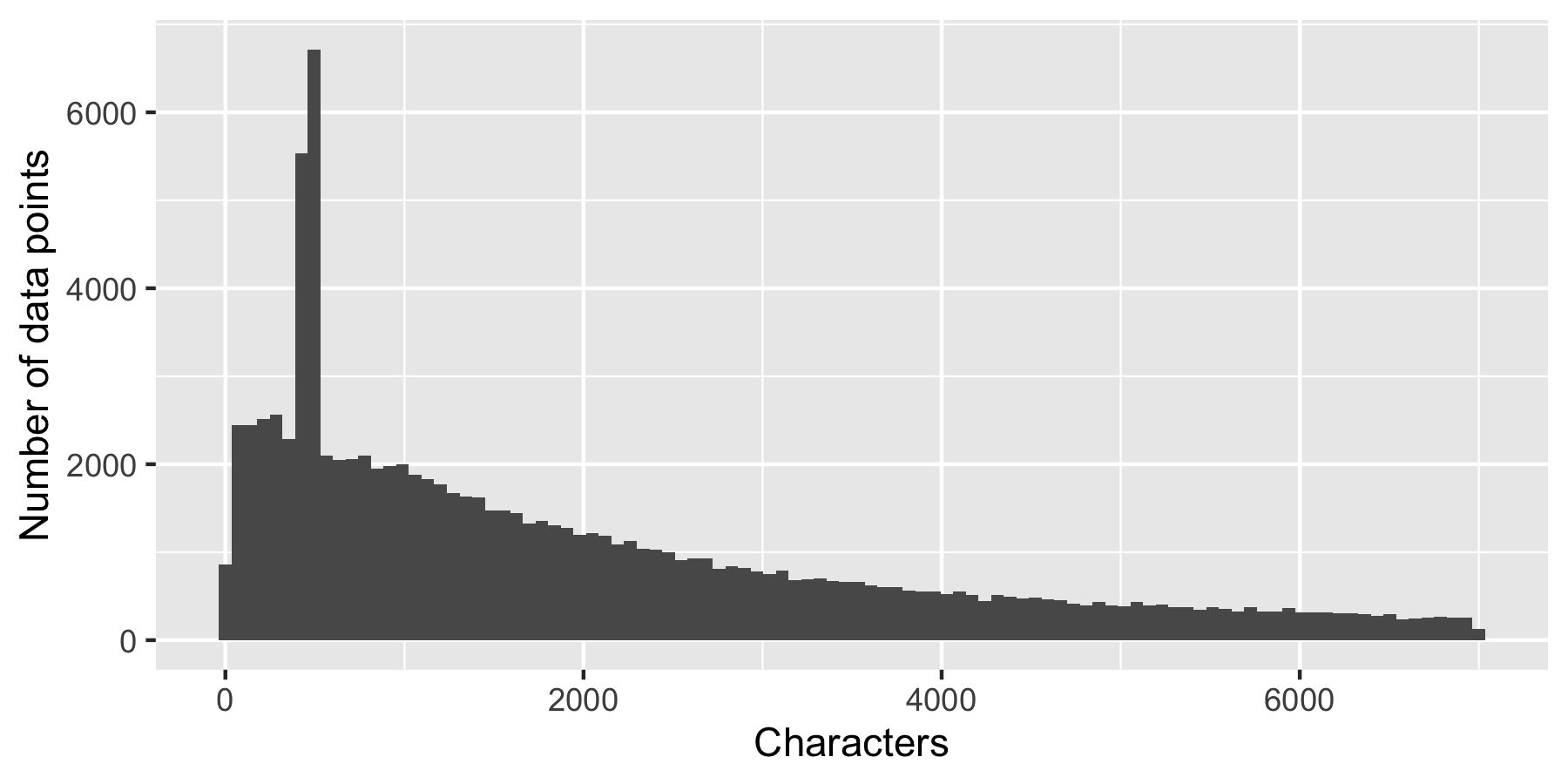}
         \caption{Python scripts in the ETH Py150 Open dataset with fewer than 7000 characters.}
     \end{subfigure}
    \caption{The number of characters in the ETH Py150 Open dataset \cite{kanade2020learning}.}
    \label{fig:py150len}
\end{figure}

\begin{figure}
    \centering
     \begin{subfigure}[b]{0.48\linewidth}
         \centering
         \includegraphics[width=\linewidth]{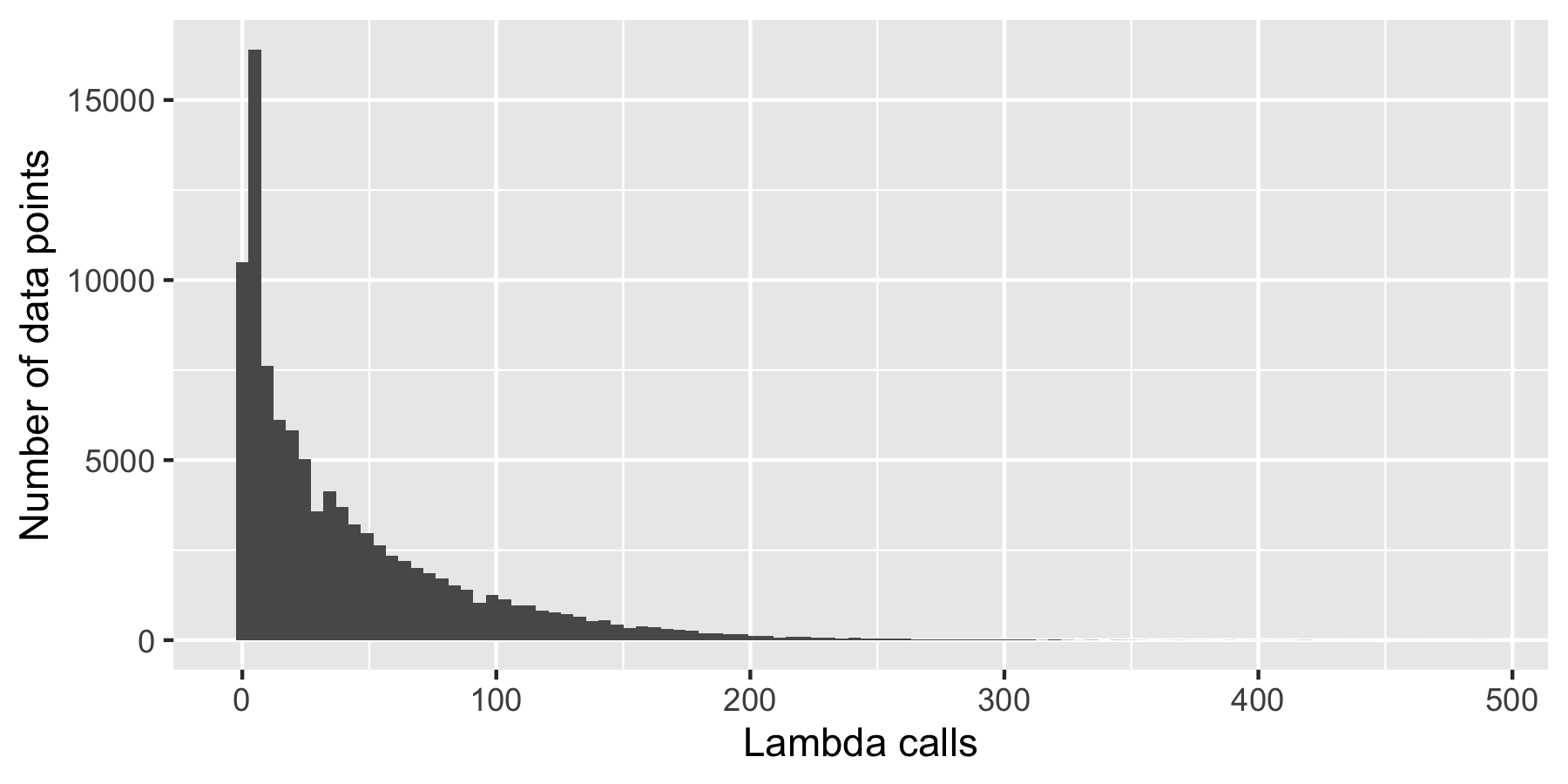}
         \caption{The entire ETH Py150 Open dataset. }
     \end{subfigure}\hfill
      \begin{subfigure}[b]{0.48\linewidth}
         \centering
         \includegraphics[width=\linewidth]{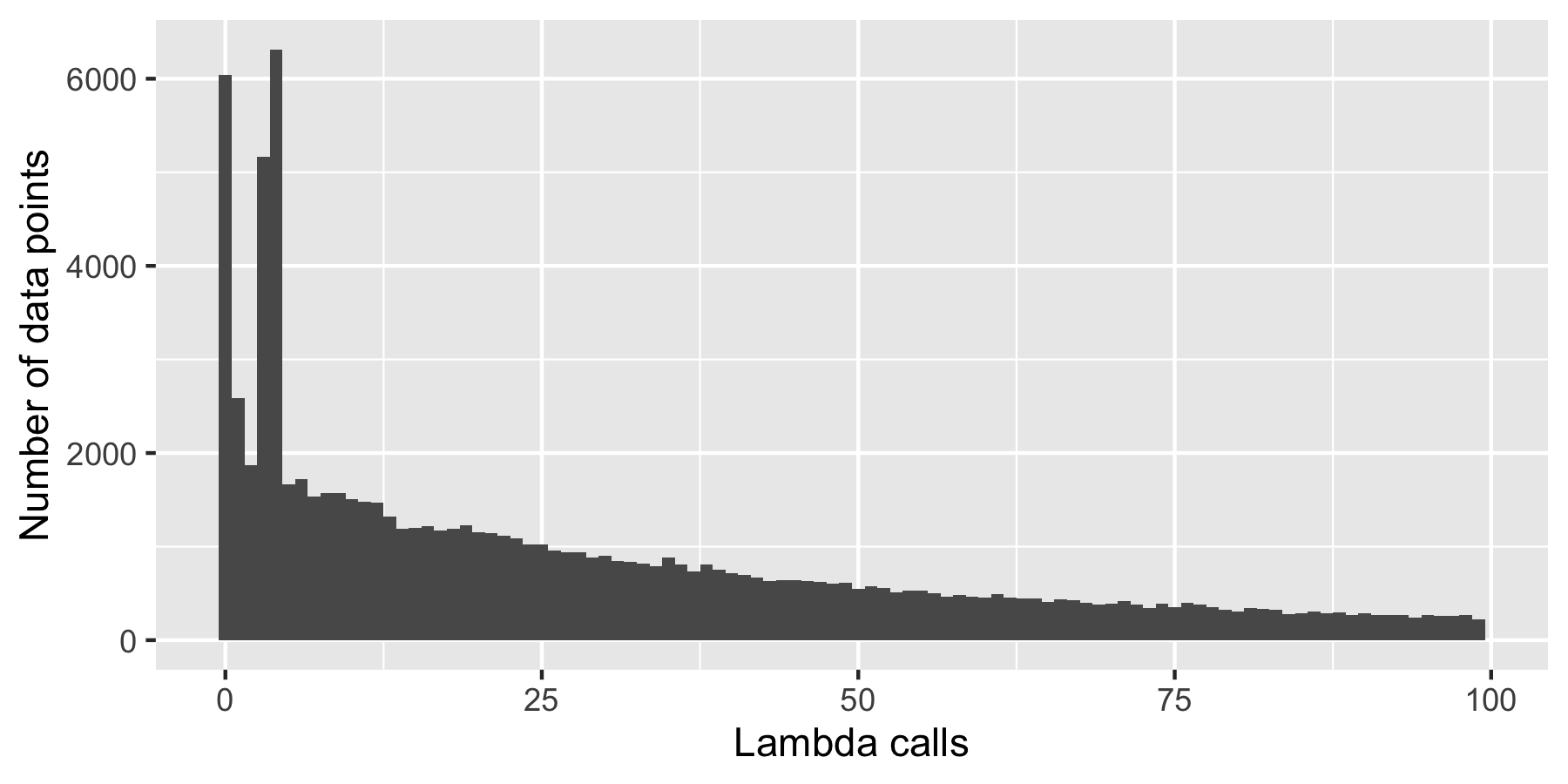}
         \caption{Python scripts in the ETH Py150 Open dataset with fewer than 100 \texttt{lambda} calls. }
     \end{subfigure}
    \caption{The number of \texttt{lambda} calls in the filtered ETH Py150 Open dataset \cite{kanade2020learning}.}
    \label{fig:py150call}
\end{figure}

\begin{figure}
    \centering
     \begin{subfigure}[b]{0.48\linewidth}
         \centering
         \includegraphics[width=\linewidth]{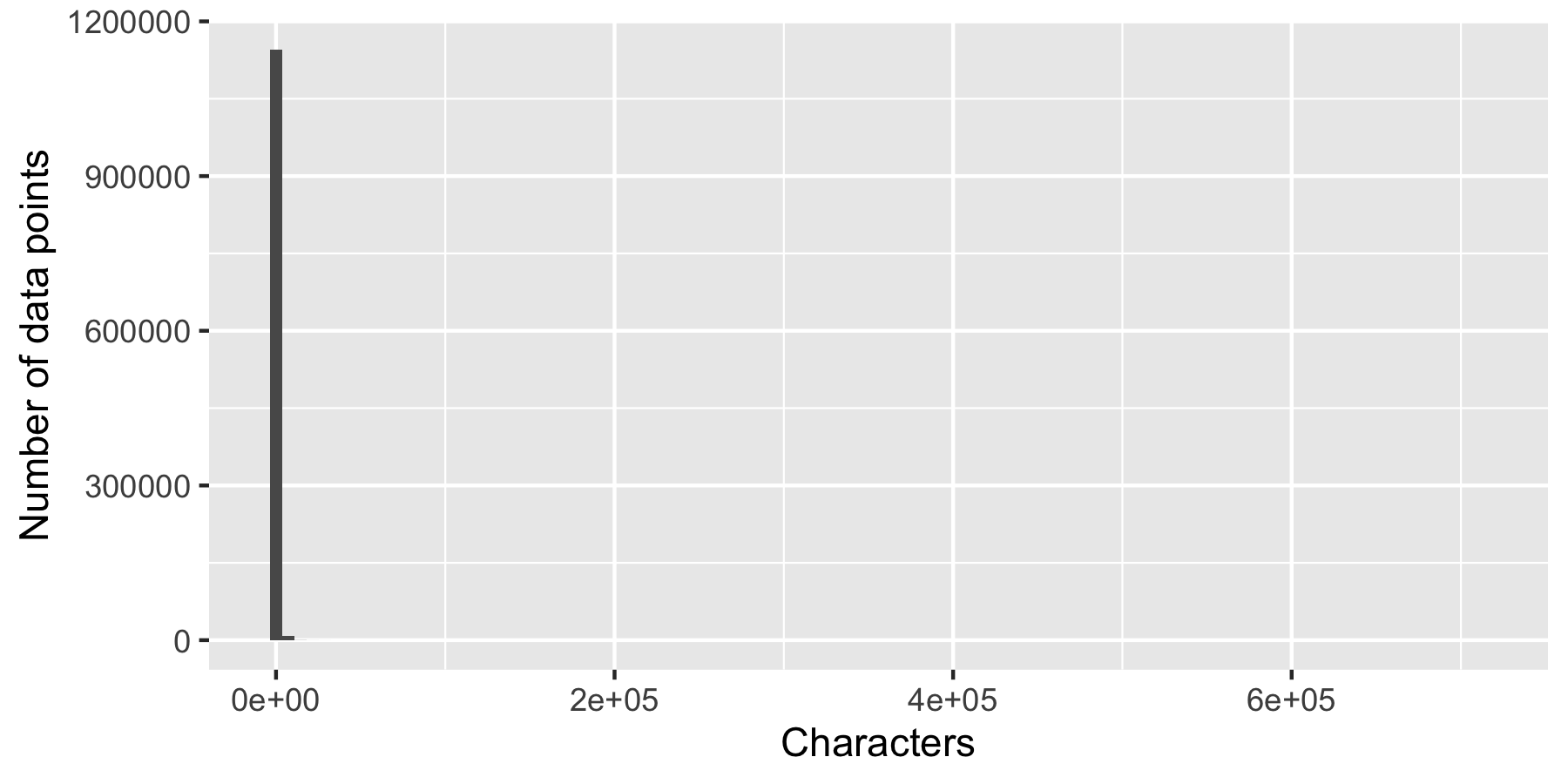}
         \caption{The entire variable misuse localization and repair dataset. }
     \end{subfigure}\hfill
      \begin{subfigure}[b]{0.48\linewidth}
         \centering
         \includegraphics[width=\linewidth]{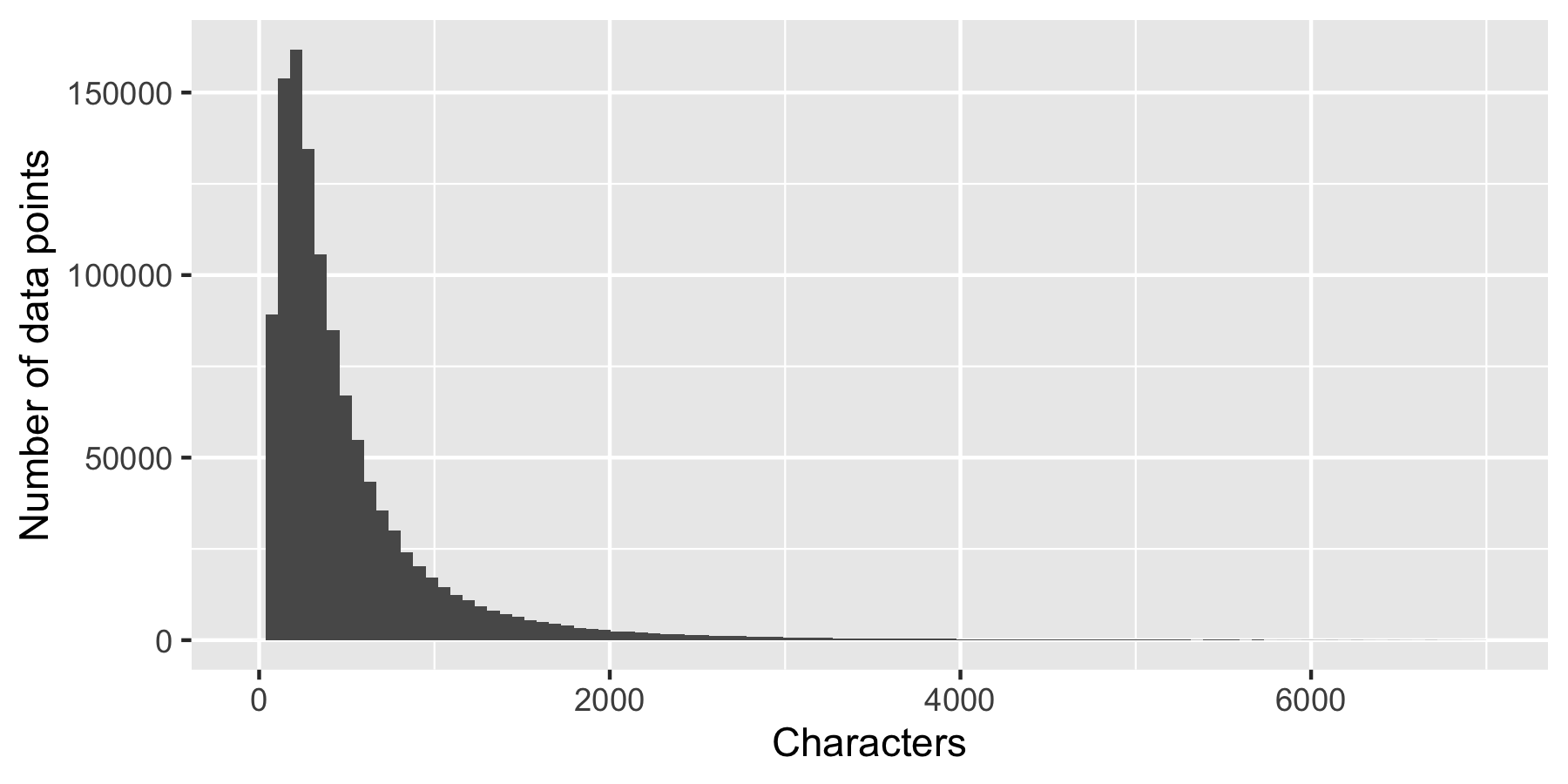}
         \caption{Python scripts in the variable misuse localization and repair dataset with fewer than 7000 characters.}
     \end{subfigure}
    \caption{The number of characters in the variable misuse localization and repair dataset \cite{kanade2020learning}.}
    \label{fig:misuselen}
\end{figure}

\begin{figure}
    \centering
     \begin{subfigure}[b]{0.48\linewidth}
         \centering
         \includegraphics[width=\linewidth]{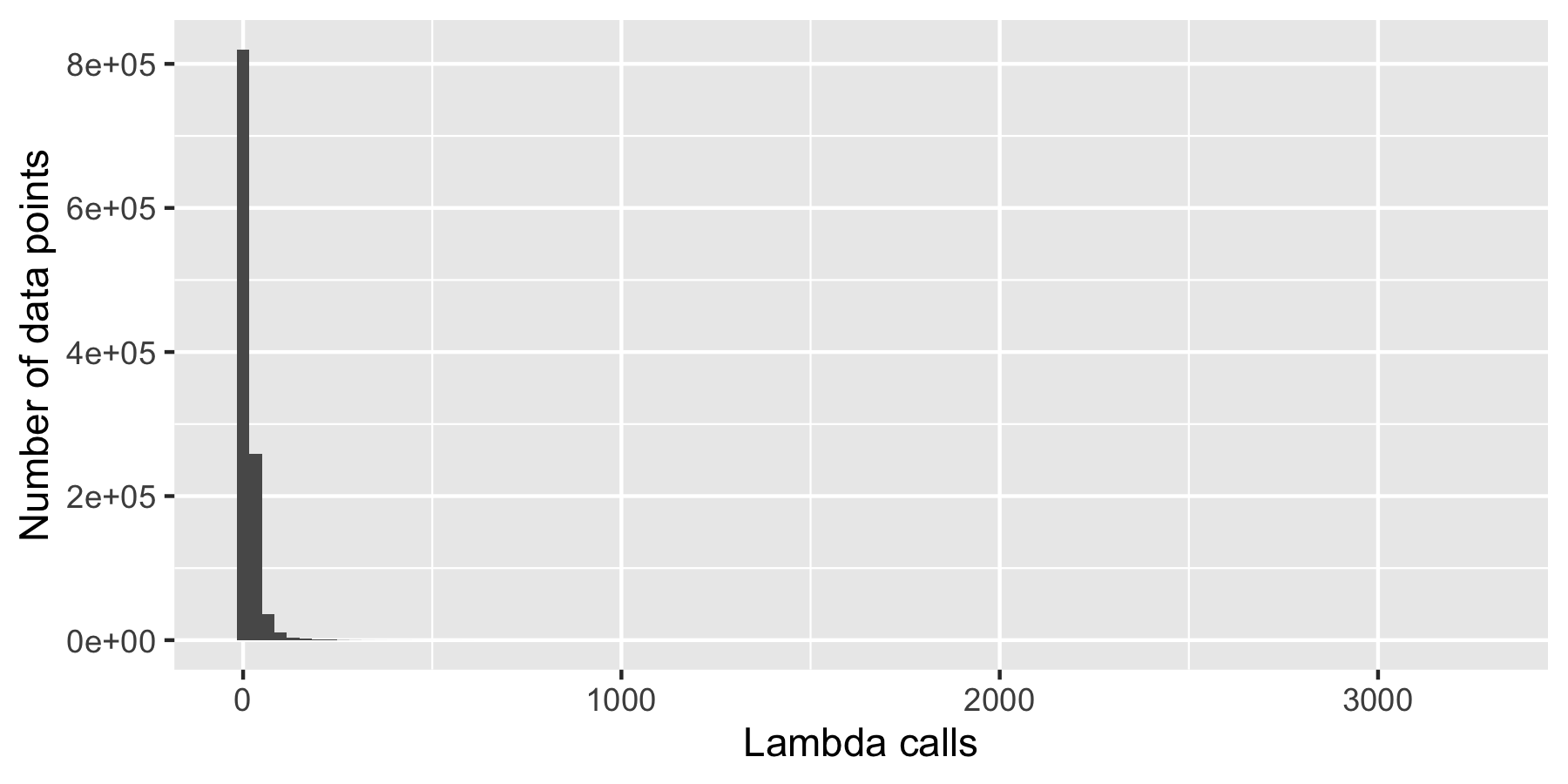}
         \caption{The entire variable misuse localization and repair dataset. }
     \end{subfigure}\hfill
      \begin{subfigure}[b]{0.48\linewidth}
         \centering
         \includegraphics[width=\linewidth]{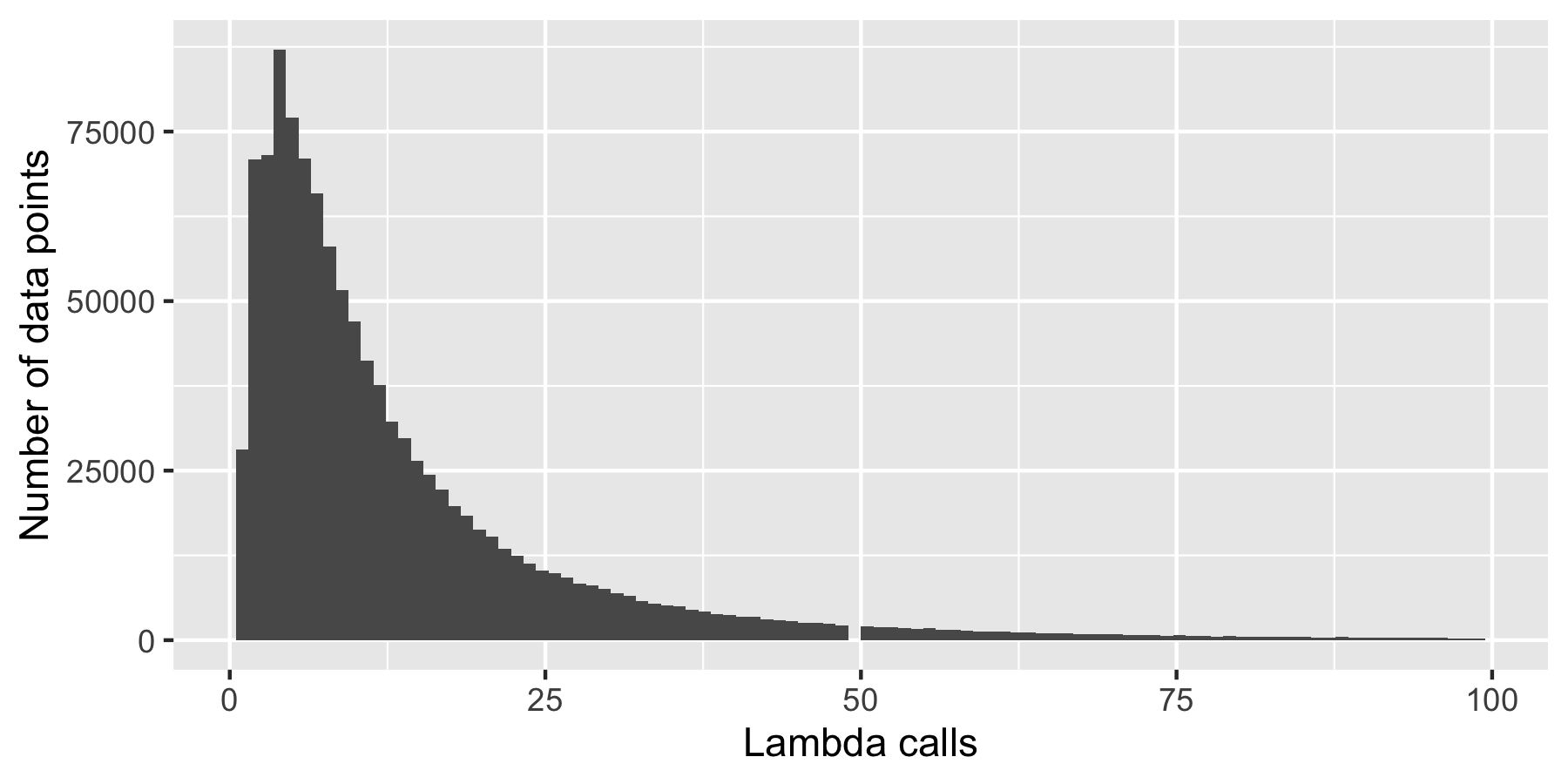}
         \caption{Python scripts in the variable misuse localization and repair dataset with fewer than 100 \texttt{lambda} calls. }
     \end{subfigure}
    \caption{The number of \texttt{lambda} calls in the filtered variable misuse localization and repair dataset \cite{kanade2020learning}.}
    \label{fig:misusecall}
\end{figure}

The filtered data set is given to the Neural Interpretation. The Guesser Transformer's input sequence is truncated to have at most 512 tokens, following the standard practice \cite{feng2020codebert}. The $\lambda$-Executor's function call is truncated to have at most 16 arguments per function call, with unlimited control flow context vectors. For experiments in Section \ref{sec:learning} ``Learning to Model the Execution of Generic Python Code'', every batch of 16 inputs have at most 128 \texttt{lambda} calls before batched execution terminates, as a form of truncation to maintain consistent peak GPU memory usage and prevent out-of-memory errors during training. For experiments in Section \ref{sec:misuse} ``White-box variable misuse localization and repair'', a batch of 64 inputs have at most 1024 \texttt{lambda} calls before batched execution terminates, as the scripts in the misuse dataset are shorter on average.

For experiments in Section \ref{sec:misuse} ``White-box variable misuse localization and repair'', the task is defined as four localization and repair steps. For this problem, the Code Generator needs to correctly label the argument in a function call to be misused and propagate the contamination. However, not all variable misuse occurs as a function call's argument variable. After executing the code, the misuse may not be labeled for any of the four steps, for which the loss could not be calculated. 
We find that around 0.5\% of the data points have such labeling errors (see Tab \ref{tab:dataset}), and their results are not included in Tab. \ref{tab:misuse} and Tab. \ref{tab:baseline}.

\section{Detailed definitions of loss objectives}
\label{appx:loss}

\subsection{Learning to Model the Execution of Generic Python Code}
\textit{Return variable classification loss $L_1$.} NI is trained to contrastively classify the return variable of an executed statement. For example, \texttt{celsius\_to\_fahrenheit(25)} is likely to be assigned to variables called \texttt{fahrenheit} or \texttt{converted}, but not \texttt{duck} or \texttt{table}. Given the executed result $r$, NI learns to select the correct variable $l$ to be assigned based on its guessed representation, among $K-1$ randomly sampled variables, through cross entropy.

For any assignment in the form of \texttt{lhs = [expression]}, neurally execute the right-hand-side expression (\texttt{[expression]}) to obtain the returned vector $r$. 
Let $\mathbf V : \mathbb R^{K\times H}$ be a matrix  with the guessed representation $l$ of the variable on the left-hand-side (\texttt{lhs}) and guessed representations of $K-1$ randomly sampled variables in the input batch of source code. Let $\alpha: \mathbb R^{2H}\to \mathbb R$ be a neural network decoder. $L_1$ is a cross entropy loss with softmax function:
\begin{align}
    L_1  = -\log \frac{\exp(\alpha(r, l))}{\sum_{i=1}^K \exp(\alpha(r, \mathbf V_i)}.
\end{align}

\textit{Argument discrimination loss $L_2$.} NI is trained to recognize a function call with randomly replaced arguments. For example, the returned abstract semantics of \texttt{celsius\_to\_fahrenheit(duck)} is unclear. NI learns to distinguish a real function call and a function call with misused arguments through binary cross entropy.

For any function call in the form of \texttt{fun(args)}, neurally execute it to obtain the return vector $r$. Randomly sample a function call \texttt{fun2(args2)} from the batch input of source code. Neurally Execute \texttt{fun(args2)} to obtain $r'$. Let $\beta: \mathbb R^H \to \mathbb R$ be a neural network decoder, and $\sigma$ be the sigmoid function. $L_2$ is a binary cross entropy loss:
\begin{align}
    L_2 = -\log \sigma(\beta(r)) - \log (1-\sigma(\beta(r'))).
\end{align}

\textit{Data-flow discrimination loss $L_3$.} NI is trained to recognize the data-flow path from one variable to another. For example, \texttt{fahrenheit} can likely be computed from data stored in variables \texttt{celsius}, \texttt{temperature}, or \texttt{thermometer}, but not \texttt{passport} or \texttt{duck}. NI learns to distinguish two variables with a data-flow path and two variables without a data-flow path through binary cross entropy.

Let $r$ be the return vector of a function call \texttt{fun(args)}. Let $a$ be the vector representing one of the arguments \texttt{args}. We define $(a,r)$ as a data-flow edge \cite{liu1998simple,ferrante1987program} and build a data-flow graph (DFG). Randomly sample a data dependency path from $b$ to $c$. Randomly sample two nodes $(b', c')$ in the same DFG that has no data dependency path from $b'$ to $c'$. Let $\phi: \mathbb R^{2H}\to \mathbb R$ be a neural network decoder. $L_3$ is a binary cross entropy loss:
\begin{align}
    L_3 = -\log \sigma(\phi(b, c)) - \log (1-\sigma(\phi(b', c'))).
\end{align}

\subsection{White-Box Variable Misuse Localization and Repair}
The white-box variable misuse localization and repair has four steps. During training, all five losses are added together with $L=\sum_{k=1}^5 L_k$.
\paragraph{Code classification.}
To train NI to predict whether the input code has a misuse, we collect the return values $r_1, r_2, ... r_L$ of all $L$ function calls. We train a one-layer BERT-style Transformer encoder $\kappa$ and compute $\kappa([r_1, r_2, ... r_L]) = o$, where $o$ is the un-normalized binary classification logit $o\in \mathbb R$. Given misuse label $y\in\{0,1\}$, binary cross entropy $L_1$ is computed:
\begin{align}
    L_1 = -y\log\sigma(o) - (1-y) \log (1-\sigma(o))
\end{align}

\paragraph{Call localization.}
To train NI to localize the first $\lambda$-Executor call with a misused variable $C$, we label $C$ to be contaminated. When a contaminated argument is used in a call executed by $\lambda$-Executor, the return value is also contaminated recursively. We collect the return values $\mathbf r = r_1, r_2, ... r_L$ of all $L$ function calls, and we have a binary contamination label $c_l$ for every $r_l$.
We train a one-layer BERT-style Transformer encoder $\eta$ and compute $\eta([r_1, r_2, ... r_L]) = [o_1, o_2, ... o_L]$, where $o_l\in \mathbb R$ for $l=1,2... L$.
A binary cross entropy loss $L_2$ is computed as follows. 
\begin{align}
    L_2 = \sum_{l=1}^L -c_l\log\sigma(o_l) - (1-c_l) \log (1-\sigma(o_l))
\end{align}

It's always true in our dataset that one of the return values $r_C \in \mathbf r$ is computed by using the misuse variable $C$ as an argument. We train a one-layer BERT-style Transformer encoder $\psi$ and compute $\psi([r_1, r_2, ... r_L]) = [s_1, s_2, ... s_L]$, and uses the output to classify the source contamination call that returns $r_C$. A cross entropy loss $L_3$ is computed as follows:
\begin{align}
    L_3 = -\log \frac{\exp(s_C)}{\sum_{i=1}^L \exp(s_i)}
\end{align}

For inference, we do not need to predict for each lambda call whether its returned value is contaminated. Instead, we can find the source contamination call with $\psi$ directly. Therefore, the task associated with $L_2$ is for training only.

\paragraph{Argument localization.} To train NI to localize the misused argument in the localized call, we train a one-layer network that takes the Transformer output vector $a'$ for each argument $a$ to classify the misused argument $a_C$. Let $a_1, a_2 ... a_N$ be the arguments to the $\lambda$ Transformer in Eq. \ref{eq:execbert}, the aligned outputs are $a'_1, a'_2, ... a_N'$. Let $\tau: \mathbb R^H\to \mathbb R$ be a one-layer network. The cross entropy loss $L_4$ for argument localization is
\begin{align}
    L_4  = -\log \frac{\exp(\tau(a'_C))}{\sum_{i=1}^N \exp(\tau(a'_i))}.
\end{align}

Note that in all other tasks, $a'$ is discarded, as we only keep the return vector $r$ (Eq. \ref{eq:return} and Eq. \ref{eq:execret}). For inference, we simply predict the argument with the maximum logit score.

\paragraph{Repair.} To train NI to repair the misuse, like a debugger, we pause when the misused call is executed to collect all vector objects $\mathbf o$ in the Interpreter memory. We substitute the misused argument with each object $o$ and obtain the return values. 

Let the misused return value with $a_C$ be 
\begin{align}
    r = \lambda([\theta_f', c_1, ... c_M, a_0, ... a_C, ... a_N])
\end{align} by Eq. \ref{eq:return} and Eq. \ref{eq:execret}. Replace the misused argument $a_C$ with an object $o\in \mathbf o$ and execute the call with the same currying signature parameter, contexts, and other arguments, we have
\begin{align}
    r_o= \lambda([\theta_f', c_1, ... c_M, a_0, ... o, ... a_N]).
\end{align} 
Let $r_\mathbf o$ be the return values of the contaminated call with replaced objects.
A one-layer network $\pi: \mathbb R^H \to \mathbb R$ is trained to classify the return value from the correct argument through cross entropy loss $L_4$:
\begin{align}
    L_5 = -\log \frac{\exp(\pi(r^*))}{\sum_{r_o\in r_{\mathbf o}} \exp(\pi(r_o))},
\end{align}
where $r^*$ is the return value of the correct $\lambda$ call
\begin{align}
    r^* = \lambda([\theta_f', c_1, ... c_M, a_0, ... o^*, ... a_N]).
\end{align}
For inference, we simply predict the object with the maximum logit score to repair the misused argument.

\section{Time and Space Complexity of NI}
\label{appx:complexity}

\begin{theorem}[Time Complexity of NI]\label{prop:time}
Discount the initialization time and assume that execution of a neural network representing a function has $\mathcal O(1)$ time complexity. NI has $\mathcal O(N)$ time complexity for source code with length $N$.
\end{theorem}
\begin{proof}
No statement (AST node) is executed twice by \mbox{NI} due to the linearity principle. A source code with length $N$ has $\mathcal O(N)$ function calls, each taking constant time.
\end{proof}

The time complexity of Neural Interpretation (Thm. \ref{prop:time}) discounts the initialization time from the Guesser, because the Guesser is chosen to be a Transformer by our implementation, which has $\mathcal O(N^2)$ complexity for source code with length $\mathcal O(N)$. In the Neural Interpretation architecture, the Transformer can be swapped out to other neural models that may be faster, so $\mathcal O(N^2)$ initialization time an engineering artifact and is not essential to NI's time complexity. Execution complexity is essential to Neural Interpretation, which is $\mathcal O(N)$. In practice, the Guesser Transformer is only run once per input, and the Executor is run every time a \texttt{lambda} call is needed, so the quadratic complexity of Guesser may be a small overhead, especially with a combination of a small Guesser Transformer and a large Executor Transformer.

The $\lambda$-Executor can execute \texttt{lambda} calls in a $\mathcal O(1)$ time, if the length of the input is bounded by a constant. That is, the number of arguments and the levels of nested control flows of a function call is $\mathcal O(1)$. For most of the Python programs in the wild, this assumption is true. 

Linear execution time complexity has the potential to be a faster execution model than Transformers, which typically have $\mathcal O(N^2)$ complexity. Of course, if the Guesser's initialization time is not discounted, then Neural Interpretation's time complexity is still bottle-necked by the Transformer's time complexity.

\begin{theorem}[Space Complexity of NI]\label{prop:space}
Considering only the Neural Interpreter's memory space, NI has $\mathcal O(N)$ space complexity for source code with length $N$.
\end{theorem}
\begin{proof}
Every representation stored in the memory is associated with a variable name, and the number of variable names is $\mathcal O(N)$ w.r.t. source code of length $N$.
\end{proof}

Neural Interpreter's memory space is all that is needed for Neuro-Symbolic Execution, so we only consider its space complexity. For engineering, Neural Interpretation may use additional space for neural parameters, training, and analysis. They are not considered in this space complexity. The engineering components that correspond to this space complexity are: 1) the Guesser Transformer's output; 2) Guesser's pooled representation for each value; 3) the last Executor output for each variable encountered. All three have $\mathcal O(N)$ space complexity.

\section{Missingness and guessing}
\label{appx:missing}

We think that missingness in generic source code is unavoidable. One example of a highly missing source code is a single statement:

\texttt{fahrenheit = (celsius * 1.8) + 32}

If this statement is executed in a real Python interpreter, we will get \texttt{NameError: name 'celsius' is not defined}. Compared to this example, the \texttt{celcius\_to\_fahrenheit} function in Fig. \ref{fig:mental} is a valid Python program that will be accepted by a real Python interpreter, because \texttt{celsius} is a parameter of a function definition.

Nevertheless, Neural Interpretation can abstractly execute this invalid statement. In fact, Neural Interpretation can execute any script of linear statements with some statements removed, as long as the script can be parsed and Code Generator recognizes the syntax. Of course, the Neuro-Symbolic Execution results would not reflect the correct semantics if some critical statements are removed.

\section{Neural network hyperparameters}
\label{appx:network}
We use the CodeBERT \cite{feng2020codebert} tokenizer and architecture, a 12-layer BERT model using Transformer Encoder. When fine-tuning we load the CodeBERT \cite{feng2020codebert} weights for initialization.

Huggingface's default parameters are used. The optimizer is AdamW \cite{loshchilov2017decoupled}. The learning rate is set to be 5e-5. A linear learning rate scheduler with 5\% warm up is used. Mixed precision training is used. Gradient check pointing is enabled and is vital to reduce $\lambda$-Executor memory consumption. 

As a first model aiming to prove feasibility, we do not tune Neural Interpretation aggressively. Validation set is only used to verify that in variable misuse call localization, filtering calls with positive predicted contamination score to construct the data-flow graph has sensible validation accuracy compared to other thresholds of contamination scores. After training of all experiments, the last checkpoint is used, and early stopping is not applied. 

As mentioned in the paper, for experiments in Section \ref{sec:learning} ``Learning to Model the Execution of Generic Python Code'', a batch size of 16 is used, which defines the sampling diversity of negative examples in the cross entropy (Appx. \ref{appx:loss}). For each of the three abstract semantics (return variable classification, argument discrimination, data-flow discrimination), the classification loss is computed for 64 samples per input batch. For experiments in Section \ref{sec:misuse} ``White-box variable misuse localization and repair'', the batch size is set to be 64.

The truncation methods and hyperparameters are discussed in Appx. \ref{appx:dataset}.

\section{Performance of Neural Interpretation}
Neural Interpretation is the first Neuro-Symbolic Execution model. It is the first model of computers that achieve the following five desiderata:
1) neural representation learning from name semantics; 2) white-box step-by-step execution; 3) execution of partial source code  without concrete inputs; 4) execution of general-domain source code; 5) composition of neural networks ``programmed'' according to the source code.

We do not emphasize performance in terms of accuracy as a desideratum to be pursued by Neural Interpretation, because there are no baseline models in the literature that can achieve the five desiderata above. They are qualitative differences that put Neural Interpretation in a different category. In our view, trading accuracy for the five desiderata is worthwhile, as neuro-symbolic reasoning is a major challenge in our path toward artificial general intelligence, and solving neuro-symbolic reasoning is not about incrementally improving accuracy with short-term results. Qualitative properties prevent deep learning methods today from reaching artificial general intelligence, and it is more important to deliver qualitative AGI features such as neuro-symbolic reasoning than improving accuracy incrementally. Neural Interpretation delivers the five desiderata, which are all qualitative features we expect to establish when moving toward AGI.

A first work in a new category may not immediately become the state-of-the-art. It often takes a whole community of efforts with iterations of research results to improve the performance. Transformers \cite{vaswani2017attention}, for example, have been improved by the whole community since its invention 5 years ago. Neural networks were invented 80 years ago \cite{mcculloch1943logical} and did not become the state-of-the-art until the recent 20 years \cite{krizhevsky2017imagenet}. 
Neural Interpretation has just been introduced in this paper, and it has great potential for future improvements.

\section{Is Neural Interpretation an interpreter or a compiler?}
\label{appx:interpreter}

Neural Interpretation is an interpreter, because it executes every statement immediately as the AST node is encountered during code generation. Although it is possible for Neural Interpretation to output the intermediate code in neural instructions (e.g. Fig. \ref{fig:simple}) as a compiler would to machine code, it does not explicitly do so because interpretation is simpler.

Neural Compilation in this paper refers to the compilation of a function definition to vector representations by executing it abstractly. Neural Compilation is compilation in the sense that it translates textual source code to an executable that is a currying neural network (Eq. \ref{eq:return}). Neural Interpretation performs Neural Compilation when a function definition is encountered during interpretation and does not separately compile them beforehand.

\section{Inductive bias and differentiability}

The Parser and Code Generator of Neural Interpretation handles the syntactical grammars, which, in a way, enforces an ``invisible'' structural prior onto the Guesser and Executor neural networks. The Executor, especially, never sees or needs to learn the grammar of Python. The syntax is taken care of by the Code Generator and Interpreter, such that $\lambda$ has a homogenous interface for all Python constructs that need to be abstractly executed. The Executor knows the semantics of a Python syntactical construct. For example, it can differentiate multiplication and addition, and it can differentiate a while-loop and a for-loop contextual vectors, because the currying signatures of these built-in functions correspond to different embeddings (see Appx. \ref{appx:build_in}). 

Although Neural Execution emulates program execution, in a traditional sense, program execution is a symbolic computation and not differentiable. Why can Neural Interpretation learn through gradient descent? The back-propagation computational graph (e.g. autograd in PyTorch) that trains the $\lambda$-Executor is similar to that of a recurrent neural network, only with a non-linear structure with branching and merging based on the data-flow specified by the source code. All operations are differentiable, and the undifferentiable syntactical constructs are handled by the other components and not learned by the $\lambda$-Executor, and that's why $\lambda$-Executor can learn differentiably with gradient descent.

\section{Built-in functions}
\label{appx:build_in}

For some built-in functions (including operators), we do not use the Guesser to initialize their abstract semantics. Instead, they are looked up from an embedding matrix, such that they have the same representations for all scripts, regardless of the context. The looked-up embedding is used as the signature parameter for currying of $\lambda$-Executor (Eq. \ref{eq:return}). These built-in functions may be built-in functions in Python as well, such as \texttt{+}, \texttt{*}, while others are specific for Neural Interpretation, such as \texttt{\_\_while\_\_} to transform the abstract semantics of the while-statement's condition to a contextual vector given to the Executor.

We list all 69 built-in functions and constants of Neural Interpretation below. The order is arbitrary, as the built-in function is dynamically assigned an incrementing embedding index when first encountered during training. Note in this list that \texttt{\_\_compile\_function\_\_} is for Neural Compilation. \texttt{\_\_while\_\_} and \texttt{\_\_for\_in\_\_} are for control flow contexts. \texttt{val\_obj\_val\_default} is to initialize the guessed representation of an object if corresponding tokens are not found for pooling. \texttt{\_\_subscript\_\_} is \texttt{\_\_list\_index\_\_} mentioned in Methods section to follow standard Python terminology.
\begin{itemize}[topsep=-5pt,itemsep=-1ex,partopsep=1ex,parsep=1.3ex, leftmargin=2ex]
\item \texttt{and} 
\item \texttt{fun\_obj\_val\_default} 
\item \texttt{val\_obj\_val\_default} 
\item \texttt{const\_obj\_tensor\_default} 
\item \texttt{\_\_try\_\_} 
\item \texttt{\_\_except\_\_} 
\item \texttt{\_\_tuple\_of\_\_} 
\item \texttt{\_\_compile\_function\_\_} 
\item \texttt{\_\_dictionary\_key\_value\_\_} 
\item \texttt{==} 
\item \texttt{<} 
\item \texttt{-} 
\item \texttt{\_\_if\_\_} 
\item \texttt{\_\_dictionary\_of\_\_} 
\item \texttt{\_\_else\_\_} 
\item \texttt{\_\_list\_of\_\_} 
\item \texttt{\%} 
\item \texttt{not} 
\item \texttt{\_\_keyword\_argument\_\_} 
\item \texttt{\_\_get\_attr\_\_} 
\item \texttt{\_\_subscript\_\_} 
\item \texttt{\_\_list\_splat\_\_} 
\item \texttt{\_\_dictionary\_splat\_\_} 
\item \texttt{+} 
\item \texttt{in} 
\item \texttt{\_\_for\_in\_\_} 
\item \texttt{\_\_default\_parameter\_\_} 
\item \texttt{+=} 
\item \texttt{\_\_end\_for\_iterator\_\_} 
\item \texttt{\_\_unpack\_k\_\_} 
\item \texttt{\_\_slice\_\_} 
\item \texttt{is} 
\item \texttt{\_\_generator\_\_} 
\item \texttt{*} 
\item \texttt{/} 
\item \texttt{<=} 
\item \texttt{>} 
\item \texttt{\_\_conditional\_expression\_\_} 
\item \texttt{or} 
\item \texttt{!=} 
\item \texttt{\_\_subscript\_assign\_\_} 
\item \texttt{>=} 
\item \texttt{\_\_expression\_list\_of\_\_} 
\item \texttt{|=} 
\item \texttt{**} 
\item \texttt{\_\_set\_of\_\_} 
\item \texttt{\_\_while\_\_} 
\item \texttt{\_\_list\_comprehension\_\_} 
\item \texttt{\_\_if\_clause\_\_} 
\item \texttt{>>} 
\item \texttt{\&} 
\item \texttt{<<} 
\item \texttt{|} 
\item \texttt{\_\_dictionary\_comprehension\_\_} 
\item \texttt{-=} 
\item \texttt{//} 
\item \texttt{\_\_finally\_\_} 
\item \texttt{*=} 
\item \texttt{\&=} 
\item \texttt{/=} 
\item \texttt{\^{}} 
\item \texttt{>>=} 
\item \texttt{\textasciitilde} 
\item \texttt{\_\_parenthesis\_\_} 
\item \texttt{<>} 
\item \texttt{<<=} 
\item \texttt{\%=} 
\item \texttt{\^{}=} 
\item \texttt{//=} 
\end{itemize}

\end{document}